\documentclass{article}

\usepackage[preprint]{neurips_2020}

\usepackage[utf8]{inputenc} 
\usepackage[T1]{fontenc}    
\usepackage{hyperref}       
\usepackage{url}            
\usepackage{booktabs}       
\usepackage{amsfonts}       
\usepackage{nicefrac}       
\usepackage{microtype}
\usepackage{xcolor}
\usepackage{amsmath,amsthm,amsfonts,amssymb,amscd}
\usepackage{enumerate}
\usepackage{multirow}
\usepackage{cleveref}
\usepackage{algorithm}
\usepackage{subfig}
\usepackage{graphicx}
\usepackage{caption}
\usepackage{xspace}
\usepackage{algpseudocode}
\newcommand{\cifar}{{CIFAR-10}\xspace}
\newcommand{\swb}{{SWB300}\xspace}
\newcommand{\sync}{{SYNC}\xspace}
\newcommand{\adpsgd}{{ADPSGD}\xspace}
\newcommand{\adpdpsgd}{{A(DP)$^2$SGD}\xspace}
\newcommand{\eff}{{EfficientNet-B0}\xspace}
\newcommand{\vgg}{{VGG-19}\xspace}
\newcommand{\resnet}{{ResNet-18}\xspace}

\newcommand{\googlenet}{{GoogleNet}\xspace}
\newcommand{\mobilenet}{{MobileNet}\xspace}
\newcommand{\mobilenetvtwo}{{MobileNetV2}\xspace}
\newcommand{\shuf}{{ShuffleNet}\xspace}
\newcommand{\resnext}{{ResNext-29}\xspace}
\newcommand{\senet}{{SENet-18}\xspace}

\usepackage{autobreak}
\allowdisplaybreaks[4]

\title{\adpdpsgd: Asynchronous Decentralized Parallel Stochastic Gradient Descent with Differential Privacy}

\author{%
  Jie Xu \\
  Weill Cornell Medicine\\
  New York, USA \\
  \texttt{jix4002@med.cornell.edu} \\
   \And
   Wei Zhang \\
   IBM Research \\
   New York, USA \\
   \texttt{weiz@us.ibm.com} \\
   \AND
   Fei Wang \thanks{Corresponding author} \\
   Weill Cornell Medicine \\
   New York, USA \\
   \texttt{few2001@med.cornell.edu} \\
}
\newtheorem{prop}{Proposition}
\newtheorem{assum}{Assumption}
\newtheorem{deff}{Definition}
\newtheorem{remark}{Remark}
\newtheorem{lemma}{Lemma}
\newtheorem{theorem}{Theorem}

\begin{document}

\maketitle

\begin{abstract}
As deep learning models are usually massive and complex, distributed learning is essential for increasing training efficiency. Moreover, in many real-world application scenarios like healthcare, distributed learning can also keep the data local and protect privacy. 
A popular distributed learning strategy is federated learning, where there is a central server storing the global model and a set of local computing nodes updating the model parameters with their corresponding data. The updated model parameters will be processed and transmitted to the central server, which leads to heavy communication costs. Recently, asynchronous decentralized distributed learning has been proposed and demonstrated to be a more efficient and practical strategy where there is no central server, so that each computing node only communicates with its neighbors. Although no raw data will be transmitted across different local nodes, there is still a risk of information leak during the communication process for malicious participants to make attacks. 
In this paper, we present a differentially private version of asynchronous decentralized parallel SGD (\adpsgd) framework, or \adpdpsgd for short, which maintains communication efficiency of \adpsgd and prevents the inference from malicious participants.
Specifically, R{\'e}nyi differential privacy is used to provide tighter privacy analysis for our composite Gaussian mechanisms while the convergence rate is consistent with the non-private version.
Theoretical analysis shows \adpdpsgd also converges at the optimal $\mathcal{O}(1/\sqrt{T})$ rate as SGD. 
Empirically, \adpdpsgd achieves comparable model accuracy as the differentially private version of Synchronous SGD (SSGD) but runs much faster than SSGD in heterogeneous computing environments.
\end{abstract}

\section{Introduction}
Distributed Deep Learning (DDL), as a collaborative modeling mechanism that could save storage cost and increase computing efficiency when carrying out machine learning tasks,
has demonstrated strong potentials in various areas, especially for training large deep learning models on large dataset such as ImageNet~\cite{akiba2017chainermn,zhang2015staleness,you2017imagenet}. 
Typically, assume there are $K$ workers where the data reside (a worker could be a machine or a GPU, \emph{etc}), distributed machine learning problem boils down to solving an empirical risk minimization problem of the form:
\begin{equation}\label{eq:dml}
\min_{\mathbf{w}\in\mathbb{R}^d} F(\mathbf{w}):=\mathbb{E}_{k \sim {\mathcal{I}}}[F_k(\mathbf{w})]=\sum_{k=1}^K p_kF_k(\mathbf{w}),
\end{equation}
where $p_k$'s define a distribution, that is, $p_k\geq 0$ and $\sum_k p_k=1$, and $p_k$ indicates the percentage of the updates performed by worker $k$. The objective $F(\mathbf{w})$ in problem~(\ref{eq:dml}) can be rephrased as a linear combination of the local empirical objectives $F_k(\mathbf{w}):=\mathbb{E}_{\xi \sim {\mathcal{D}_k}}[f(\mathbf{w};\xi)]$, where $\mathcal{D}_k$ denotes the data distribution associated to worker $k\in[K]$ and $\xi$ is a data point sampled via $\mathcal{D}_k$.

In particular, algorithms for DDL face with following issues. On one hand, the communication cost to the central server may not be affordable since a large number of updates of a number of workers are usually involved. Many practical peer-to-peer networks are usually dynamic, and it is not possible to regularly access a fixed central server. Moreover, because of the dependency on the central server, all workers are required to agree on one trusted central body, and whose failure would interrupt the entire training process for all workers. Therefore, researchers have started to study fully decentralized framework where the central server is not required~\cite{shayan2018biscotti,roy2019braintorrent,Lalitha2019decentralized,lalitha2019peer}, which is also the focus of this paper. In addition, to improve flexibility and scalability, as in~\cite{lian2018asynchronous}, we consider the asynchronous communication where the participant workers do not operate in the lock-step.

On the other hand, since a large number of workers usually participate in the training process in distributed learning, it is difficult to ensure none of them are malicious. Despite no raw data sharing and no central body are required to coordinate the training process of the global model, the open computing network architecture and extensive collaborations among works still inevitably provide the opportunities for malicious worker to infer the private information about another worker given the execution of $f(\mathbf{w})$, or over the shared predictive model $\mathbf{w}$~\cite{truex2018hybrid}. To alleviate this issue, differential privacy (DP), as an alternative theoretical model to provide mathematical privacy guarantees, has caught people's attention~\cite{dwork2006our}.
DP ensures that the addition or removal of a single data sample does not substantially affect the outcome of any analysis, thus is widely applied to many algorithms to prevent implicit leakage, not only for traditional algorithms, \emph{e.g.} 
principal component analysis~\cite{chaudhuri2013near}, support vector machine~\cite{rubinstein2009learning}, but also for modern deep learning research~\cite{abadi2016deep,mcmahan2017learning}. 

In this paper, we focus on achieving differential privacy in asynchronous decentralized communication setting, where we target to obtain a good convergence rate while keeping the communication cost low. We highlight the following aspects of our contributions:
\begin{itemize}
	\item We propose a differentially private version of \adpsgd, \emph{i.e.}, \adpdpsgd, 
	where differential privacy is introduced to protect the frequently exchanged variables.
	\item We present the privacy and utility guarantees for \adpdpsgd, where R{\'e}nyi differential privacy is introduced to provide tighter privacy analysis of composite heterogeneous mechanisms~\cite{mironov2017renyi} while the convergence rate is consistent with the non-private version.
	\item Empirically, we conduct experiments on both computer vision (CIFAR-10) and speech recognition (\swb) datasets. \adpdpsgd achieves comparable model accuracy and level of DP protection as differentially private version of Synchronous SGD (SSGD) and runs much faster than SSGD in heterogeneous computing environments. 
\end{itemize}



\section{Related Work}
\subsection{Differential Privacy} 
Differential privacy (DP), first introduced by Dwork \emph{et al.} \cite{dwork2006our}, is a mathematical definition for the privacy loss associated with any data release drawn from a statistical database. The basic property of DP mechanism is that the change of the output probability distribution is limited when the input of the algorithm is slightly disturbed. Formally, it says:
\begin{deff}[($\epsilon,\delta$)-DP  \cite{dwork2006our}]
	A randomized mechanism $\mathcal{M}:\mathcal{X}^n\to\mathcal{R}$ satisfies ($\epsilon,\delta$)-differential privacy, or ($\epsilon,\delta$)-DP for short, if for all $\mathbf{x}, \mathbf{x}'\in\mathcal{X}^n$ differing on a single entry, and for any subset of outputs $S \subseteq \mathcal{R}$, it holds that
	\begin{equation}
	Pr[\mathcal{M}(\mathbf{x})\in S]\leq e^{\epsilon}Pr[\mathcal{M}(\mathbf{x}')\in S] + \delta.
	\end{equation}
\end{deff}

The parameter $\epsilon$ balances the accuracy of the differentially private $\mathcal{M}$ and how much it leaks \cite{shokri2015privacy}. The presence of a non-zero $\delta$ allows us to relax the strict relative shift in unlikely events \cite{dwork2006our}. 

Although by relaxing the guarantee to ($\epsilon,\delta$)-DP, advanced composition allows tighter analyses for compositions of (pure) differentially private mechanisms ($\delta=0$), iterating this process quickly leads to a combinatorial explosion of parameters  \cite{mironov2017renyi}. To address the shortcomings of ($\epsilon,\delta$)-DP, Mironov \emph{et al.}  \cite{mironov2017renyi} proposed a natural relaxation of differential privacy based on the R{\'e}nyi divergence, \emph{i.e.}, R\'{e}nyi differential privacy (RDP). 
The new definition RDP shares many important properties with the standard definition of differential privacy, while additionally allowing for a more rigorous analysis of composite heterogeneous mechanisms \cite{mironov2017renyi}.



\subsection{Differentially Private Distributed Learning}
Existing literature on differentially private distributed learning either focus on centralized learning or synchronous communication or convex problems. Our work 
combines decentralized learning, asynchronous communication and non-convex optimization in a DP setting. In contrast, Cheng \emph{et al.} \cite{cheng2018leasgd,cheng2019towards} focus on decentralized learning systems and aim to achieve differential privacy, but their convergence analysis is applied to strongly convex problems only. Bellet \emph{et al.} \cite{bellet2017fast,bellet2017personalized} obtain an efficient and fully decentralized protocol working in an asynchronous fashion by a block coordinate descent algorithm and make it differentially private with Laplace mechanism, but their convergence analysis is only applied to convex problems. 

Lu \emph{et al.} \cite{lu2019differentially} propose a differentially private asynchronous federated learning scheme for resource sharing in vehicular networks. They perform the convergence boosting by updates verification and weighted aggregation without any theoretical analysis. 
Li \emph{et al.} \cite{li2019asynchronous} aims to secure asynchronous edge-cloud collaborative federated learning with differential privacy. But they choose centralized learning and conduct analysis under the convex condition. In this paper, we focus on achieving differential privacy in asynchronous decentralized communication setting and dealing with non-convex problems.

\section{\adpdpsgd Algorithm}





Specifically, the decentralized communication network is modeled via an undirected graph $G=([K],\,\mathcal{E})$, consisting of the set $[K]$ of nodes and the set $\mathcal{E}$ of edges. The set $\mathcal{E}$ are unordered pairs of elements of $[K]$. Node $k\in [K]$ is only connected to a subset of the other nodes, and not necessarily all of them. We allow the information collected at each node to be propagated throughout the network.
To improve flexibility and scalability, we consider the asynchronous communication where the participant workers do not operate in a lock-step~\cite{lian2018asynchronous}.

During optimization, each worker maintains a local copy of the optimization variable. Suppose that all local models are initialized with the same initialization, \emph{i.e.}, $\mathbf{w}_k^0=\mathbf{w}^0, k=1,...,K$. Let $\mathbf{w}_k^t$ denote the value at worker $k$ after $t$ iterations. We implement stochastic gradient descent in a decentralized asynchronous manner by the following steps, which are executed in parallel at every worker, $k=1,...,K$:
\begin{itemize}
	\item \textbf{Sample data}: Sample a mini-batch of training data denoted by $\{\xi_k^{i}\}_{i=1}^B$ from local memory of worker $k$ with the sampling probability $\frac{B}{n_{k}}$, where $B$ is the batch size.
	\item \textbf{Compute gradients}: Worker $k$ locally computes the stochastic gradient:
	$g^t(\hat{\mathbf{w}}_{k}^t;\xi_{k}^t):=\sum_{i=1}^B\triangledown F_{k}(\hat{\mathbf{w}}_{k}^t;\xi_{k}^{t,i})$,
	where $\hat{\mathbf{w}}_{k}^t$ is read from the local memory.
	\item \textbf{Averaging}: Randomly sample a doubly stochastic matrix
	$\mathbf{A}$ and average local models by:
	\begin{equation}\label{eq:update}
	[\mathbf{w}_1',\mathbf{w}_2',...,\mathbf{w}_K']\leftarrow[\mathbf{w}_1,\mathbf{w}_2,...,\mathbf{w}_K]\mathbf{A};
	\end{equation}
	Note that each worker runs the above process separately without any global synchronization.
	\item \textbf{Update model}: Worker $k$ locally updates the model:
	\begin{equation}\label{eq:update2}
	\mathbf{w}_{k} \leftarrow \mathbf{w}_{k}'-\eta g^t(\hat{\mathbf{w}}_{k};\xi_{k}),
	\end{equation}
	Noth that \textbf{averaging} step and \textbf{update model} step can run in parallel. $\hat{\mathbf{w}}_{k}$ may not be the same as $\mathbf{w}'_{k}$ since it may be modified by other workers in the last \textbf{averaging} step.
\end{itemize}
All workers simultaneously run the procedure above.

As we stated, the model is trained locally without revealing the input data or the model's output to any workers, thus it prevents the direct leakage while training or using the model. However, recall in the \textbf{averaging} step, the model variables exchange frequently during training. In this case, the workers still can infer some information about another worker's private dataset given the execution over the shared model variables~\cite{truex2018hybrid}. To solve this issue, we apply differential privacy to the exchanged model variables.

The general idea to achieve differential privacy is to add a stochastic component to the variables that need to be protected. In our case, the exchanged information is model variables $\mathbf{w}_k$. 
Note that the computation of $\mathbf{w}_k$ depends on the gradients. 
Thus, instead of adding noise directly on the exchanged model variable  $\mathbf{w}_k$, we inject the noise on the gradients:
\begin{equation*}
	\tilde{g}(\hat{\mathbf{w}}_{k};\xi_{k}) = g(\hat{\mathbf{w}}_{k};\xi_{k}) + \mathbf{n},
\end{equation*}
where $\mathbf{n}\sim \mathcal{N}(0,\sigma^2\triangle_2^2(g))$ is the Gaussian distribution. The global sensitivity estimate $\triangle_2(g)$ is expected significantly reduced, resulting in higher accuracy by ensuring the norm of all gradients is bounded for each update - either globally, or locally~\cite{shokri2015privacy}.

Then the \textbf{update model} step (\ref{eq:update2}) turns into:
$\mathbf{w}_{k} \leftarrow \mathbf{w}_{k}'-\eta \tilde{g}^t(\hat{\mathbf{w}}_{k};\xi_{k})$.
Differential privacy ensures that the addition or removal of a data sample does not substantially affect the outcome of any analysis. 
The specific procedures are summarized in Algorithm~\ref{alg:gaussian}.

\begin{algorithm}[htbp]
	\setlength{\abovecaptionskip}{-0.5cm}
	\caption{\adpdpsgd (logical view)}\label{alg:gaussian}
	\begin{algorithmic}[1]
		\State \textbf{Initialization}: Initialize all local models $\{\mathbf{w}_k^0\}_{k=1}^K\in\mathbb{R}^d$ with $\mathbf{w}^0$, learning rate $\eta$, batch size $B$, privacy budget $(\epsilon, \delta)$, and total number of iterations $T$.
		\State \textbf{Output}: $(\epsilon, \delta)$-differentially private local models.
		\For{\texttt{<$t = 0,1,...,T-1$>}}
		\State Randomly sample a worker $k^t$ of the graph $G$ and randomly sample an doubly stochastic averaging matrix $\mathbf{A}_t\in\mathbb{R}^{K\times K}$ dependent on $k^t$;
		\State Randomly sample a batch $\xi_{k^t}^t:=(\xi_{k^t}^{t,1},\xi_{k^t}^{t,2},...,\xi_{k^t}^{t,B})\in\mathbb{R}^{d\times B}$
		from local data of the $k^t$-th worker with the sampling probability $\frac{B}{n_{k^t}}$;
		\State Compute stochastic gradient $g^t(\hat{\mathbf{w}}_{k^t}^t;\xi_{k^t}^t)$ locally:
		$g^t(\hat{\mathbf{w}}_{k^t}^t;\xi_{k^t}^t):=\sum_{i=1}^B\triangledown F_{k^t}(\hat{\mathbf{w}}_{k^t}^t;\xi_{k^t}^{t,i})$;
		\State Add noise
		$\tilde{g}^t(\hat{\mathbf{w}}_{k^t}^t;\xi_{k^t}^t) = g^t(\hat{\mathbf{w}}_{k^t}^t;\xi_{k^t}^t) + \mathbf{n}$,
		where $\mathbf{n}\in\mathbb{R}^d\sim \mathcal{N}(0, \sigma^2\mathbf{I})$ and $\sigma$ is defined in Theorem~\ref{the:privacy}.
		\State Average local models by $
		[\mathbf{w}_1^{t+1/2},\mathbf{w}_2^{t+1/2},...,\mathbf{w}_K^{t+1/2}]\leftarrow[\mathbf{w}_1^t,\mathbf{w}_2^t,...,\mathbf{w}_K^t]\mathbf{A}_t$;
		\State Update the local model:
		$\mathbf{w}_{k^t}^{t+1} \leftarrow \mathbf{w}_{k^t}^{t+1/2}-\eta  \tilde{g}^t(\hat{\mathbf{w}}_{k^t}^t;\xi_{k^t}^t);\quad
		\forall j\neq k^t, \mathbf{w}_j^{t+1} \leftarrow \mathbf{w}_j^{t+1/2}$.
		\EndFor
	\end{algorithmic}
\end{algorithm}

\section{Theoretical Analysis}
In this section, we present the utility and privacy guarantees for \adpdpsgd. R{\'e}nyi differential privacy is introduced to provide tighter privacy analysis of composite heterogeneous mechanisms ~\cite{mironov2017renyi} while the convergence rate is consistent with \adpsgd.

\subsection{Utility Guarantee}
We make the following assumptions which are commonly used and consistent with the non-private version of \adpsgd ~\cite{mironov2017renyi} to present the utility guarantee. 

\begin{assum}\label{assum:gradient}
	Assumptions for stochastic optimization.
	\begin{itemize} 
		\item [1)] (Unbiased Estimation). 
		$\mathbb{E}_{\xi\sim\mathcal{D}_k}[\triangledown f(\mathbf{w};\xi)]=\triangledown F_k(\mathbf{w}),
		\mathbb{E}_{k\sim\mathcal{I}}[\triangledown
		F_k(\mathbf{w})]=\triangledown F(\mathbf{w})$.
		\item [2)] (Bounded Gradient Variance).
		\begin{equation}
		\begin{aligned}
		\mathbb{E}_{\xi\sim\mathcal{D}_k}\|\triangledown f(\mathbf{w};\xi)-\triangledown F_k(\mathbf{w})\|^2\leq \varsigma^2, \mathbb{E}_{k\sim\mathcal{I}}\|\triangledown F_k(\mathbf{w})-\triangledown F(\mathbf{w})\|^2 \leq \upsilon^2.
		\end{aligned}
		\end{equation}
	\end{itemize}
\end{assum}

\begin{assum}\label{assum:asyn}
	Assumptions for asynchronous updates.
	\begin{itemize}
		\item[1)] (Spectral Gap). There exists a $\rho\in[0,1)$ such that
		\begin{equation}
		\max\{|\lambda_2(\mathbb{E}[\mathbf{A}_t^{\top}\mathbf{A}_t])|,|\lambda_K(\mathbb{E}[\mathbf{A}_t^{\top}\mathbf{A}_t])|\}\leq\rho, \forall t,
		\end{equation}
		where $\lambda_i(\cdot)$ denotes the $i$-th largest eigenvalue of a matrix.
		\item[2)] (Independence). All random variables: $k, k^t, \xi^t\in \{0,1,2,...\}$ are independent. Doubly stochastic averaging matrix $\mathbf{A}_t\in\mathbb{R}^{K\times K}$ is a random variable dependent on $k^t$.
		\item[3)] (Bounded Staleness). 
		Let's denote $\hat{\mathbf{W}}^t=\mathbf{W}^{t-\tau_t}$ and there exists a constant $\tau$ such that $\max_{t} \tau_{t}\leq \tau$.
	\end{itemize}
\end{assum}
Note that a smaller $\rho$ means faster information propagation in the network, resulting in faster convergence.

\begin{theorem}
	\label{the:utility}
	Suppose all functions $f_i(\cdot)$'s are with L-Lipschitz continuous gradients, and each of $K$ workers has dataset $D^{(k)}$ of size $n_k$. Under Assumptions~\ref{assum:gradient} and \ref{assum:asyn}, if we choose $C_1>0$, $C_2\geq 0$ and $C_3\leq 1$,
	\begin{small}
		\begin{align}
			\frac{\sum_{t=0}^{T-1}\mathbb{E}\left\| \triangledown F(\theta^t)\right\|^2}{T}
			\leq \frac{2\left(\mathbb{E}F(\mathbf{w}^0)-\mathbb{E}F^{*}\right)K}{\eta TB}+\frac{2\eta L}{BK}(\varsigma^2B+6\upsilon^2B^2+d\sigma^2),\notag
		\end{align}
	\end{small}
	where $\theta^{t}$ denotes the average of all local models at $t$-th iteration, i.e., $\theta^{t}=\frac{1}{K}\sum_{k=1}^K \mathbf{w}_k^{t}$. And $C_1, C_2, C_3$ are respectively defined as 
	\begin{small}
		\begin{align}
			C_1:=& 1-24\eta^2B^2L^2\left(\tau\frac{K-1}{K}+\bar{\rho}\right), C_3:=\frac{1}{2}+\frac{\eta BL\tau^2}{K}+\left(6\eta^2B^2L^2+\eta KBL+\frac{12\eta^3B^3L^3\tau^2}{K}\right)\frac{2\bar{\rho}}{C_1},\notag\\
			C_2:=&-\left(\frac{\eta BL^2}{K} +\frac{6\eta^2B^2L^3}{K^2}+\frac{12\eta^3B^3L^4\tau^2}{K^3}\right)\frac{4\eta^2B^2\left(\tau\frac{K-1}{K}+\bar{\rho}\right)}{C_1} +\frac{\eta B}{2K}-\frac{\eta^2B^2L}{K^2}-\frac{2\eta^3B^3L^2\tau^2}{K^3}, \notag
		\end{align}
	\end{small}
	where $\bar{\rho}=\frac{K-1}{K}\left(\frac{1}{1-\rho}+\frac{2\sqrt{\rho}}{(1-\sqrt{\rho})^2}\right)$.
\end{theorem}

Note that 
$\theta^0=\frac{1}{K}\sum_{k=1}^K\mathbf{w}_k^0=\mathbf{w}^0$ and $F^*$ denotes the optimal solution to (\ref{eq:dml}). Theorem~\ref{the:utility} describes the convergence of the average of all local models. By appropriately choosing the learning rate, we obtain the following proposition.

\begin{prop}\label{prop:rate}
	In Theorem~\ref{the:utility}, if the total number of iterations is sufficiently large, in particular,
	\begin{small}
		\begin{align}
			T\geq&L^2K^2
			\max\biggr\{192\left(\tau\frac{K-1}{K}+\bar{\rho}\right), 
			1024K^2\bar{\rho}^2,\frac{64\tau^2}{K^2}, \frac{(K-1)^{1/2}}{K^{1/6}}\left(8\sqrt{6}\tau^{2/3}+8 \right)^2\left(\tau+\bar{\rho}\frac{K}{K-1}\right)^{2/3}\biggr\},\notag
		\end{align}
	\end{small}
	and we choose learning rate $\eta=\frac{K}{B\sqrt{T}}$, then we obtain the following convergence rate
	\begin{small}
		\begin{align}\label{eq:prop1}
			\frac{\sum_{t=0}^{T-1}\mathbb{E}\left\| \triangledown F(\theta^t)\right\|^2}{T}
			\leq \frac{2(F(\mathbf{w}^0)-F^{*})+2L(\varsigma^2/B+6\upsilon^2+d\sigma^2/B^2)}{\sqrt{T}}.\notag
		\end{align}
	\end{small}
\end{prop}

Proposition~\ref{prop:rate} indicates that if the total number of iterations is sufficiently large, the convergence rate of  \adpdpsgd is $\mathcal{O}(1/\sqrt{T})$ which is consistent with the convergence rate of \adpsgd. This observation indicates that the differentially private version inherits the strengths of \adpsgd.

\subsection{Privacy Guarantee}
\begin{theorem}[Privacy Guarantee]\label{the:privacy}
	Suppose all functions $f_i(\cdot)$'s are $G$-Lipschitz and each of $K$ workers has dataset $D^{(k)}$ of size $n_k$. Given the total number of iterations $T$, for any $\delta>0$ and privacy budget $\epsilon\leq 10B^2T\alpha/(3K^2n_{(1)}^2\mu)$, \adpdpsgd with injected Gaussian noise $\mathcal{N}(0,\sigma^2\mathbf{I})$ is $(\epsilon,\delta)$-differentially private with $\sigma^2=20G^2T\alpha/(K^2n_{(1)}^2\mu\epsilon)$, where $\alpha=\log(1/\delta)/((1-\mu)\epsilon)+1$, if there exits $\mu\in(0,1)$ such that
	\begin{small}
		\begin{equation}
		\alpha\leq\log\left(\frac{K^3n_{(1)}^3\mu\epsilon}{K^2n_{(1)}^2\mu\epsilon B+5T\alpha B^3}\right),
		\end{equation}
	\end{small}
	where  $n_{(1)}$ is the size of the smallest dataset among the $K$ workers.
\end{theorem}

Additionally, we observe that differential privacy is also guaranteed for each intermediate model estimator of each worker:
\begin{remark}
	At each iteration $t\in[T]$, intermediate model estimator $\mathbf{w}_{k}^t$ is $(\sqrt{t/T}\epsilon, \delta)$-differentially private, $k\in[K]$.
\end{remark}

Recall Theorem~\ref{the:utility}, the difference of introducing Gaussian mechanism lies on the term $2L(d\sigma^2/B^2)/\sqrt{T}$ compared to \adpsgd. The Gaussian noise injected in each iteration is proportional to the total number of iterations. That is, to achieve differential privacy, we need to pay for a constant term which is proportional to the added noise at each iteration.
By assuming all functions $f_i(\cdot)$'s are with G-Lipschitz and plugging the noise level into the Proposition~\ref{prop:rate}, we obtain the following Proposition.

\begin{prop}[Utility Guarantee]\label{prop:pay}
	Suppose all functions $f_i(\cdot)$'s are G-Lipschitz in Proposition~\ref{prop:rate}. Given $\epsilon,\delta>0$, under the same conditions of Theorem~\ref{the:privacy}, if the number of iterations $T$ further satisfies 
	\begin{small}
		\begin{equation}
		T=\frac{2\left(F(\mathbf{w}^0)-F^{*}+L(\varsigma^2/B+6\upsilon^2)\right)K^2n_{(1)}^2\epsilon^2}
		{40dLG^2\log(1/\delta)},
		\end{equation}
	\end{small}
	let $C_4=4\sqrt{5}\left(1+\frac{1}{B^2\mu(1-\mu)}\right)$, then \adpdpsgd's output $\tilde{\theta}=\sum_{t=1}^T\theta^t$ satisfies
	\begin{small}
		\begin{align}
			&\mathbb{E}\left\| \triangledown F(\tilde{\theta})\right\|^2 
			\leq C_4\frac{G\sqrt{dL\left(F(\mathbf{w}^0)-F^{*}+L(\varsigma^2/B+6\upsilon^2)\right)\log(1/\delta)}}{Kn_{(1)}\epsilon}.\notag
		\end{align}
	\end{small}
\end{prop}

\section{Experiments}
We implement Synchronous SGD (\sync) as the "golden" baseline to examine if \adpdpsgd can achieve the best possbile model accuracy while maintaining DP protection as no existing FL training method has been proven to outperform \sync for the final model accuracy. In \sync, we place an ``allreduce'' (sum) call after each learner's weight update in each iteration and then take the average of the weights across all the learners. 
We leave detailed system design and implementation to Appendix~\ref{sec:design}.

\label{sec:meth}
\subsection{Dataset and Model}
We evaluate on two deep learning tasks: computer vision and speech recognition. For computer vision task, we evaluate on \cifar dataset~\cite{krizhevsky2009learning} with 9 representative convolutional neural network (CNN) models~\cite{pytorch-cifar}: \shuf~\cite{shufflenet}, \mobilenetvtwo~\cite{mobilenetv2}, \eff~\cite{efficientnet}, \mobilenet~\cite{mobilenet}, \googlenet~\cite{googlenet}, \resnext~\cite{resnext}, \resnet~\cite{resnet}, \senet~\cite{senet}, \vgg~\cite{vgg}. Among these models, ShuffleNet, MobileNet(V2), EfficientNet represent the low memory footprint models that are widely used on mobile devices, where federated learnings is often used. The other models are standard CNN models that aim for high accuracy.

For speech recognition task, we evaluate on \swb dataset. The acoustic model is a long short-term memory (LSTM) model with 6 bi-directional layers. Each layer contains 1,024 cells (512 cells in each direction). On top of the LSTM layers, there is a linear projection layer with 256 hidden units, followed by a softmax output layer with 32,000 (i.e. 32,000 classes) units corresponding to context-dependent HMM states. The LSTM is unrolled with 21 frames and trained with non-overlapping feature subsequences of that length.  The feature input is a fusion of FMLLR (40-dim), i-Vector (100-dim), and logmel with its delta and double delta (40-dim $\times$3). This model contains over 43 million parameters and is about 165MB large. 

\label{sec:results}
\subsection{Convergence Results}
\begin{table*}[]\small
	\setlength{\belowcaptionskip}{-0.5cm}
	\centering
	\begin{tabular}{|c|c|c|c|c|c|c|c|c|}
		\hline
		\multirow{2}{*}{Model/Dataset} & \multicolumn{2}{c|}{Baseline} & \multicolumn{2}{c|}{Noise (Small)} & \multicolumn{2}{c|}{Noise (Medium)} & \multicolumn{2}{c|}{Noise (Large)} \\ \cline{2-9} 
		& SYNC          & ADPSGD        & SYNC            & A(DP)$^2$*          & SYNC         & A(DP)$^2$   & SYNC          & A(DP)$^2$ \\ \hline
		\eff                 & 90.41        & 91.21        & 90.26          & 89.90          & 88.13           & 87.01          & 82.99          & 82.47          \\ \hline
		\resnext                        & 92.82        & 94.17        & 91.63          & 91.52          & 89.30           & 88.61          & 84.35          & 82.20          \\ \hline
		\mobilenet                      & 91.87        & 92.80        & 90.58          & 90.59          & 88.92           & 88.13          & 84.16          & 83.11          \\ \hline
		\mobilenetvtwo                    & 94.29        & 94.14        & 92.61          & 92.13          & 90.93           & 90.45          & 86.52          & 84.83          \\ \hline
		\vgg                            & 92.95        & 92.73        & 91.21          & 91.03          & 88.27           & 87.89          & 82.80          & 81.78          \\ \hline
		\resnet                         & 94.01        & 94.97        & 91.67          & 91.64          & 89.08           & 88.43          & 83.40          & 81.01          \\ \hline
		\shuf                     & 92.74        & 92.67        & 91.23          & 90.78          & 89.39           & 88.71          & 85.08          & 82.67          \\ \hline
		\googlenet                      & 94.04        & 94.65        & 91.94          & 92.26          & 90.28           & 90.05          & 86.34          & 85.51          \\ \hline
		\senet                        & 94.19        & 94.68        & 91.99          & 91.92          & 89.99           & 88.99          & 10.00          & 82.90          \\ \hline
		LSTM                           & 1.566         & 1.585         & 1.617           & 1.627           & 1.752            & 1.732           & 1.990            & 2.010            \\ \hline
	\end{tabular}
	\caption{Convergence Comparison. \cifar model utility is measured in test accuracy. \swb model utility is measured in held-out loss. Noise level ($\sigma$) for \cifar is set as 1 (small),  2 (medium), 4 (large). Noise level ($\sigma$) for SWB is set as 0.08 (small), 0.16 (medium), 0.32 (large). *A(DP)$^2$ stands for \adpdpsgd.}
	\label{tab:conv}
\end{table*}

We train each \cifar model with a batch size of 256 per GPU (total batch size 4096 across 16 GPUs) and adopt this learning rate setup: 0.4 for the first 160 epochs, 0.04 between epoch 160 and epoch 240, and 0.004 for the remaining 60 epochs. For the \swb model, we adopt the same hyper-parameter setup as in~\cite{icassp19}: we train the model with a batch size of 128 per GPU (total batch size 2048 across 16 GPUs), the learning rate linearly warmup w.r.t each epoch from 0.1 to 1 for the first 10 epochs and then anneals by a factor of $\sqrt{2}/{2}$ each epoch for the remaining 10 epochs. The Baseline column in \Cref{tab:conv} records the utility (test accuracy for \cifar and held-out loss for \swb)  when no noise is injected for \sync and \adpsgd.

The remaining columns in \Cref{tab:conv} summarize the convergence comparison between \sync and \adpsgd under various levels of noise. Baseline \adpsgd can outperform \sync due to \adpsgd's intrinsic noise can lead model training to a better generalization\cite{interspeech19} when batchsize is large. 
\Cref{fig:cifar-convergence} visualizes the convergence comparison of three models between \sync and \adpsgd on \cifar. More results are provided in Appendix~\ref{sec:add_exp}. 

\noindent \underline{Summary} \adpsgd and \sync achieve the comparable level of utility under the same level of noise injection, thus ensuring the same level of differential privacy budget. 

\subsection{Deployment in the Wild}
Federated learning is most often deployed in a heterogeneous environment where different learners are widespread across different type of network links and run on different types of computing devices. We compare \sync and \adpsgd for various noise levels (i.e. differential privacy budgets) in 3 case studies. More results are provided in Appendix~\ref{sec:add_exp}.

\paragraph{Case I: Random learner slowdown}
In this scenario, during iteration there is a random learner that runs 2X slower than normal. This could happen when some learner randomly encounters a system hiccup (e.g., cache misses). In \sync, every learner must wait for the slowest one thus the whole system slows down by a factor of 2. In contrast, \adpsgd naturally balances the workload and remains largely undisturbed.
\Cref{fig:cifar-2x-slow} illustrates the convergence w.r.t runtime comparison between \sync and \adpsgd on \cifar when a random learner is slowed down by 2X in each iteration for medium-level noise (for the sake of brevity, we omit the comparison for other levels of noise, as they exhibit similar behaviors).

\begin{figure}
	\centering
	{\includegraphics[width=1\columnwidth]{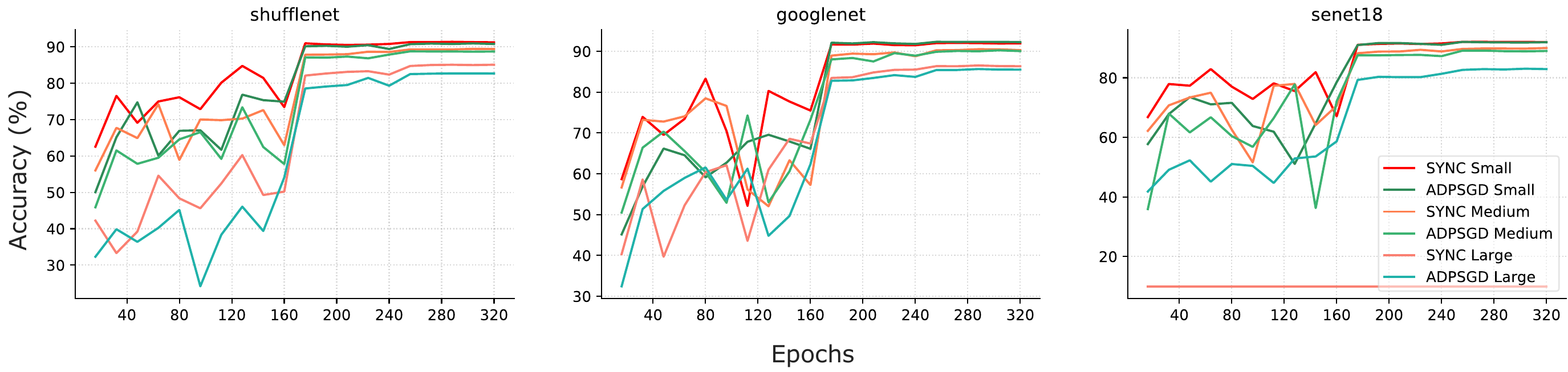}
	}
	\caption{\cifar convergence comparison between \sync and \adpsgd under various levels of noise injection (i.e., differential privacy budget). \sync and \adpsgd achieve similar level of utilities (i.e. test accuracy). }
	\label{fig:cifar-convergence}
\end{figure}

\paragraph{Case II: One very slow learner}
In this scenario, one learner is 10X slower than all the other learners. This could happen when one learner runs on an outdated device or the network links that go into the learners are low-speed compared to others. Similar to Case I, in \sync, all the learners wait for the slowest one and in \adpsgd, the workload is naturally re-balanced. \Cref{fig:cifar-10x-slow1} illustrates the convergence w.r.t runtime comparison of three models between \sync and \adpsgd on \cifar when 1 learner is slowed down by 10X in each iteration for medium-level noise.

\paragraph{Case III: Training with Large Batch}
To reduce communication cost, practitioners usually prefer to train with a larger batch size. When training with a larger batch-size, one also needs to scale up the learning rate to ensure a faster convergence~\cite{interspeech19, facebook-1hr,zhang2015staleness}. It was first reported in~\cite{interspeech19} that \sync training could collapse for \swb task when batch size is large and learning rate is high whereas \adpsgd manages to converge. We found some models (e.g., EfficientNet-B0) for computer vision task also exhibit the same trend. We increase batch size per GPU by a factor of two for \cifar and \swb, scale up the corresponding learning rate by 2 and introduce small level of noise, \sync collapses whereas \adpsgd still converges.
Figure \ref{fig:2x-bs} shows when batch size is 2X large, \sync training collapses for \cifar (EfficientNet-B0 model) and \swb , but \adpsgd manages to converge for both models.

\begin{figure}
	\setlength{\belowcaptionskip}{-0.3cm}
	\centering
	{\includegraphics[width=1\columnwidth]{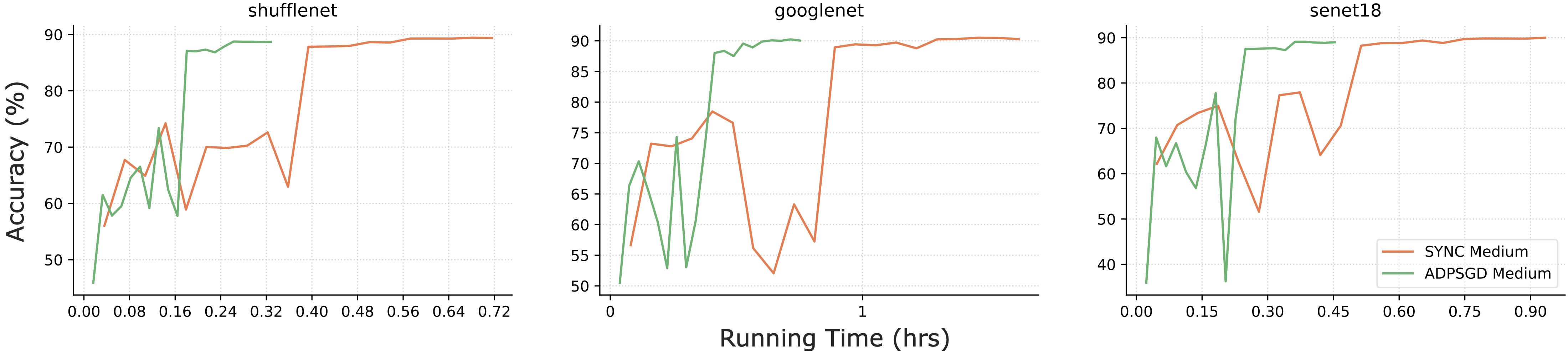}
	}
	\caption{\cifar convergence when a random learner is slowed down by 2X in each iteration with medium level of noise injection.}
	\label{fig:cifar-2x-slow}
\end{figure}

\begin{figure}
	\setlength{\belowcaptionskip}{-0.5cm}
	{\includegraphics[width=1\columnwidth]{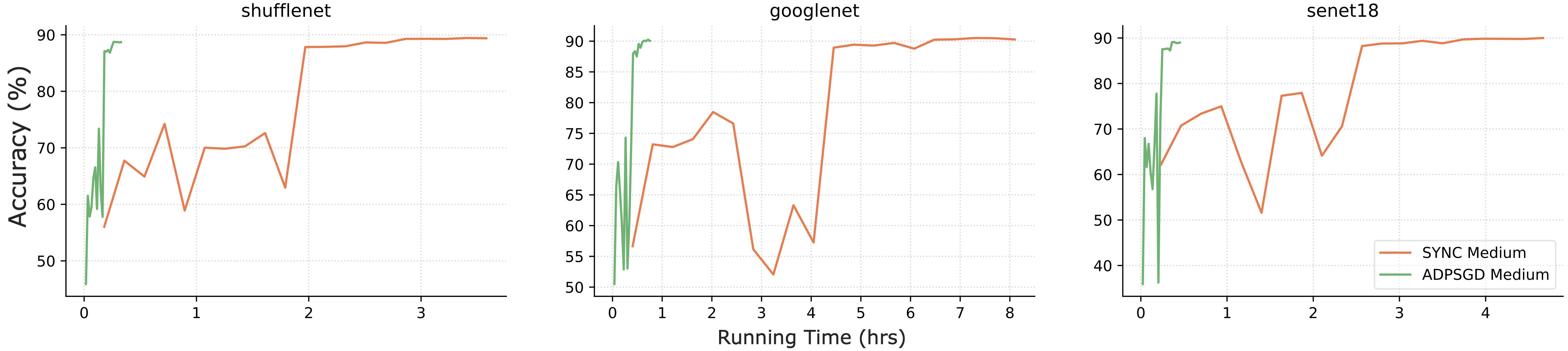}
	}
	\caption{\cifar convergence when one learner is slowed down by 10X in each iteration with medium level of noise injection. 
		\adpsgd runs significantly faster than \sync due to its asynchronous nature. }
	\label{fig:cifar-10x-slow1}
\end{figure}



\vspace{-0.2in}
\begin{figure}[htpb]
	\centering
	\subfloat[{\cifar}]
	{\includegraphics[width=0.3\columnwidth]{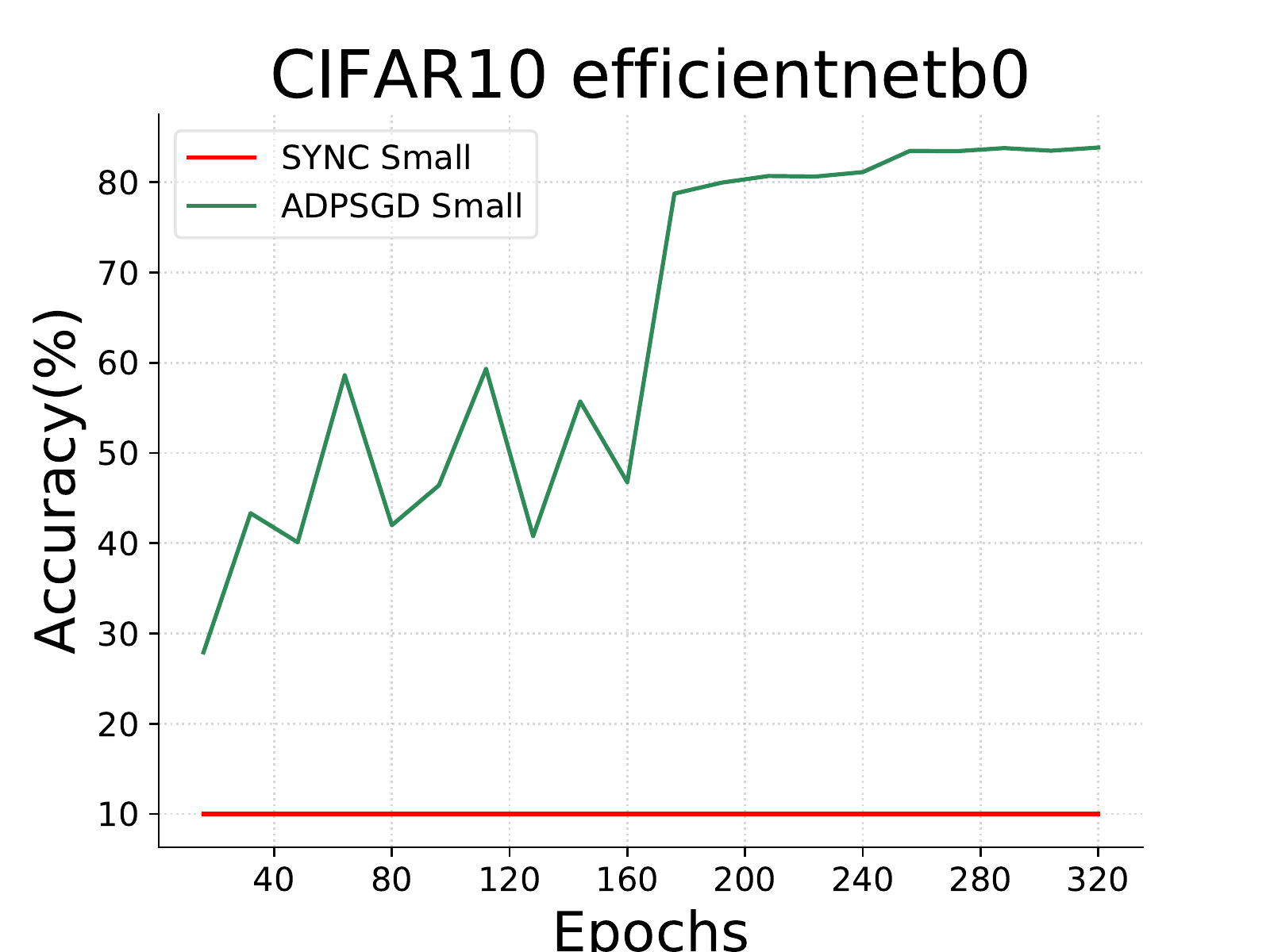}
		\label{fig:cifar-2x-bs}
	}
	\hspace{0.2in}
	\subfloat[{\swb}]
	{\includegraphics[width=0.3\columnwidth]{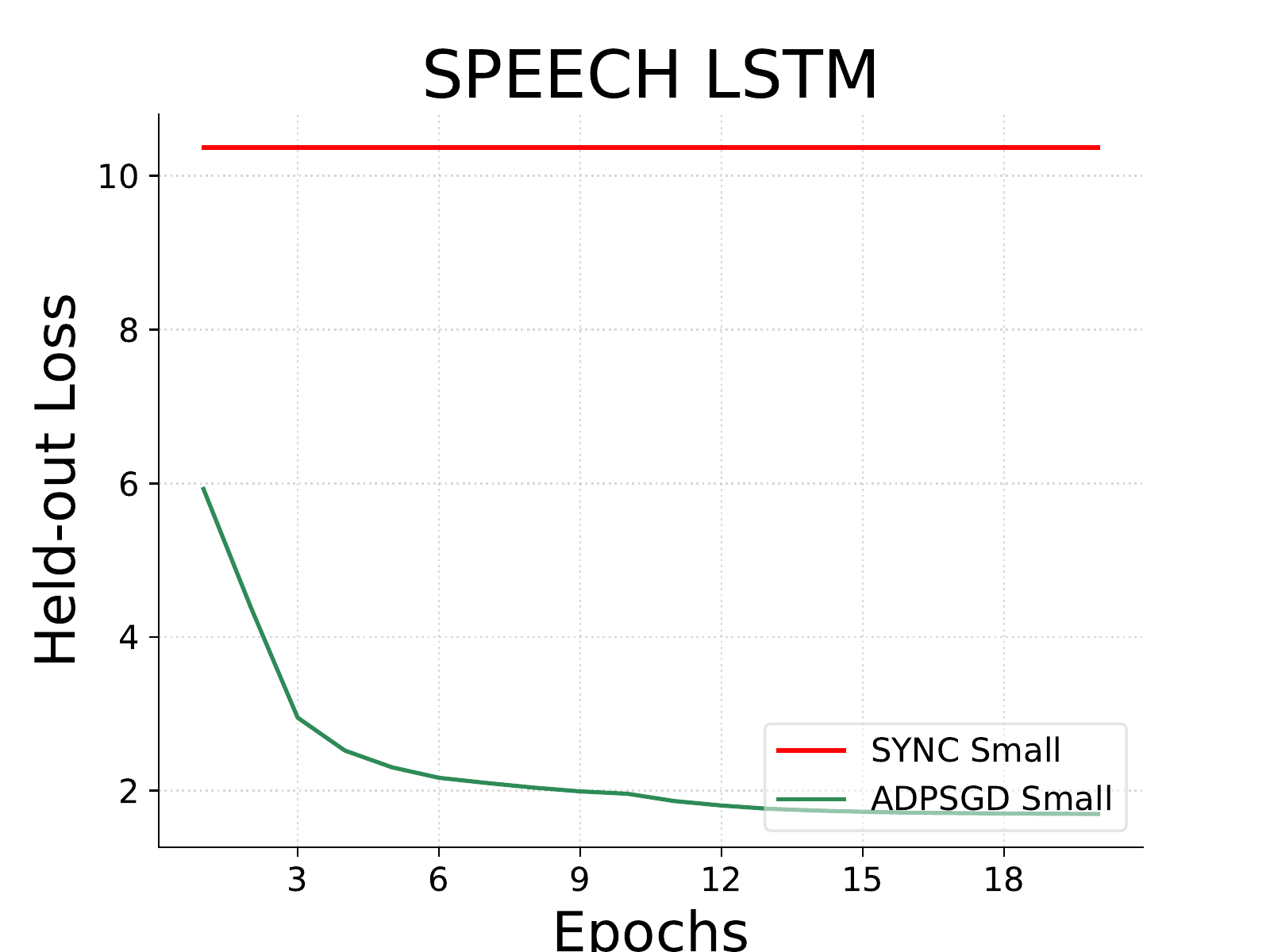}
		\label{fig:swb-2x-bs}
	}
	\caption{When batch size is 2X larger and learning rate is 2X larger, \adpsgd converges whereas \sync does not, with small level of noise injection. }
	\label{fig:2x-bs}
\end{figure}

%

\noindent \underline{Summary} In the heterogeneous environment where federated learning often operates in, \adpsgd achieves much faster convergence than \sync and \adpsgd can perform well in the case when \sync does not even converge. Since SYNC often yields the most accurate model among all DDL algorithms, we can safely conclude \adpsgd also achieves the best possible model accuracy in a DP setting, at a much higher speed.

\section{Conclusion}
This paper presents a differentially private version of asynchronous decentralized parallel SGD, maintaining communication efficiency and preventing inference from malicious participants at the same time. We theoretically analyze the impact of  DP mechanism on the convergence of \adpsgd. Our analysis shows \adpdpsgd also converges at the optimal $\mathcal{O}(1/\sqrt{T})$ rate as SGD. Besides, the privacy and utility guarantees are provided where R{\'e}nyi differential privacy is introduced to provide tighter privacy analysis for the composite Gaussian mechanism. Finally, we evaluate \adpdpsgd on both computer vision and speech recognition tasks, and the results demonstrate that it can achieve comparable utility than differentially private version of synchronous SGD but much faster than its synchronous counterparts in heterogeneous computing environments. Future work will focus on adaptive noise level associated with gradient clipping under asynchronous decentralized setting.


\section*{Broader Impact Statement}
This paper studies the problem of developing differentially private asynchronous decentralized parallel stochastic gradient descent (ADPSGD). Privacy and security are essential problems in real world applications involving sensitive information such as healthcare and criminal justice. The decentralized optimization setting can greatly enhance the security of the model training process and protect the data privacy at different locations. However, the parameter sharing and transmission process can still leak sensitive information. The DP mechanism further protects this potentially vulnerable step and makes the entire process more secure. The proposed A(DP)$^2$SGD mechanism can be broadly applied in the training process of a variety set of machine learning models, and also enhance their trustworthiness in applications involving sensitive information.



\medskip

{
\small
\bibliographystyle{plain}
\bibliography{ref}

\begin{thebibliography}{10}

\bibitem{abadi2016deep}
Martin Abadi, Andy Chu, Ian Goodfellow, H~Brendan McMahan, Ilya Mironov, Kunal
  Talwar, and Li~Zhang.
\newblock Deep learning with differential privacy.
\newblock In {\em Proceedings of the 2016 ACM SIGSAC Conference on Computer and
  Communications Security}, pages 308--318. ACM, 2016.

\bibitem{akiba2017chainermn}
Takuya Akiba, Keisuke Fukuda, and Shuji Suzuki.
\newblock Chainermn: scalable distributed deep learning framework.
\newblock {\em arXiv preprint arXiv:1710.11351}, 2017.

\bibitem{bellet2017fast}
Aur{\'e}lien Bellet, Rachid Guerraoui, Mahsa Taziki, and Marc Tommasi.
\newblock Fast and differentially private algorithms for decentralized
  collaborative machine learning.
\newblock 2017.

\bibitem{bellet2017personalized}
Aur{\'e}lien Bellet, Rachid Guerraoui, Mahsa Taziki, and Marc Tommasi.
\newblock Personalized and private peer-to-peer machine learning.
\newblock {\em arXiv preprint arXiv:1705.08435}, 2017.

\bibitem{chaudhuri2013near}
Kamalika Chaudhuri, Anand~D Sarwate, and Kaushik Sinha.
\newblock A near-optimal algorithm for differentially-private principal
  components.
\newblock {\em The Journal of Machine Learning Research}, 14(1):2905--2943,
  2013.

\bibitem{revisit-sync-sgd}
Jianmin Chen, Rajat Monga, Samy Bengio, and Rafal Jozefowicz.
\newblock Revisiting distributed synchronous sgd.
\newblock In {\em International Conference on Learning Representations Workshop
  Track}, 2016.

\bibitem{cheng2018leasgd}
Hsin-Pai Cheng, Patrick Yu, Haojing Hu, Feng Yan, Shiyu Li, Hai Li, and Yiran
  Chen.
\newblock Leasgd: an efficient and privacy-preserving decentralized algorithm
  for distributed learning.
\newblock {\em arXiv preprint arXiv:1811.11124}, 2018.

\bibitem{cheng2019towards}
Hsin-Pai Cheng, Patrick Yu, Haojing Hu, Syed Zawad, Feng Yan, Shiyu Li, Hai Li,
  and Yiran Chen.
\newblock Towards decentralized deep learning with differential privacy.
\newblock In {\em International Conference on Cloud Computing}, pages 130--145.
  Springer, 2019.

\bibitem{dwork2006our}
Cynthia Dwork, Krishnaram Kenthapadi, Frank McSherry, Ilya Mironov, and Moni
  Naor.
\newblock Our data, ourselves: Privacy via distributed noise generation.
\newblock In {\em Annual International Conference on the Theory and
  Applications of Cryptographic Techniques}, pages 486--503. Springer, 2006.

\bibitem{facebook-1hr}
Priya Goyal, Piotr Doll{\'{a}}r, Ross~B. Girshick, Pieter Noordhuis, Lukasz
  Wesolowski, Aapo Kyrola, Andrew Tulloch, Yangqing Jia, and Kaiming He.
\newblock Accurate, large minibatch {SGD:} training imagenet in 1 hour.
\newblock {\em CoRR}, abs/1706.02677, 2017.

\bibitem{resnet}
Kaiming He, Xiangyu Zhang, Shaoqing Ren, and Jian Sun.
\newblock Deep residual learning for image recognition.
\newblock {\em CVPR}, 2015.

\bibitem{mobilenet}
Andrew~G. Howard, Menglong Zhu, Bo~Chen, Dmitry Kalenichenko, Weijun Wang,
  Tobias Weyand, Marco Andreetto, and Hartwig Adam.
\newblock Mobilenets: Efficient convolutional neural networks for mobile vision
  applications.
\newblock {\em CoRR}, abs/1704.04861, 2017.

\bibitem{senet}
Jie Hu, Li~Shen, and Gang Sun.
\newblock Squeeze-and-excitation networks.
\newblock {\em CVPR}, abs/1709.01507, 2018.

\bibitem{krizhevsky2009learning}
Alex Krizhevsky and Geoffrey Hinton.
\newblock Learning multiple layers of features from tiny images.
\newblock {\em Computer Science Department, University of Toronto, Tech. Rep},
  1(4):7, 2009.

\bibitem{lalitha2019peer}
Anusha Lalitha, Osman~Cihan Kilinc, Tara Javidi, and Farinaz Koushanfar.
\newblock Peer-to-peer federated learning on graphs.
\newblock {\em rXiv preprint arXiv:1901.11173}, 2019.

\bibitem{Lalitha2019decentralized}
Anusha Lalitha, Xinghan Wang, Osman Kilinc, Yongxi Lu, Tara Javidi, and Farinaz
  Koushanfar.
\newblock Decentralized bayesian learning over graphs.
\newblock page arXiv preprint arXiv:1905.10466, 2019.

\bibitem{li2019asynchronous}
Yanan Li, Shusen Yang, Xuebin Ren, and Cong Zhao.
\newblock Asynchronous federated learning with differential privacy for edge
  intelligence.
\newblock {\em arXiv preprint arXiv:1912.07902}, 2019.

\bibitem{lian2018asynchronous}
X~Lian, W~Zhang, C~Zhang, and J~Liu.
\newblock Asynchronous decentralized parallel stochastic gradient descent.
\newblock {\em Proceedings of the ICML 2018}, 2018.

\bibitem{pytorch-cifar}
Kang Liu.
\newblock {\em Train CIFAR10 with PyTorch}.
\newblock Available at \url{https://github.com/kuangliu/pytorch-cifar}.

\bibitem{lu2019differentially}
Yunlong Lu, Xiaohong Huang, Yueyue Dai, Sabita Maharjan, and Yan Zhang.
\newblock Differentially private asynchronous federated learning for mobile
  edge computing in urban informatics.
\newblock {\em IEEE Transactions on Industrial Informatics}, 2019.

\bibitem{mcmahan2017learning}
H~Brendan McMahan, Daniel Ramage, Kunal Talwar, and Li~Zhang.
\newblock Learning differentially private recurrent language models.
\newblock {\em arXiv preprint arXiv:1710.06963}, 2017.

\bibitem{mironov2017renyi}
Ilya Mironov.
\newblock R{\'e}nyi differential privacy.
\newblock In {\em 2017 IEEE 30th Computer Security Foundations Symposium
  (CSF)}, pages 263--275. IEEE, 2017.

\bibitem{nccl}
Nvidia.
\newblock {\em NCCL: Optimized primitives for collective multi-GPU
  communication}.
\newblock Available at \url{https://github.com/NVIDIA/nccl}.

\bibitem{roy2019braintorrent}
Abhijit~Guha Roy, Shayan Siddiqui, Sebastian P{\"o}lsterl, Nassir Navab, and
  Christian Wachinger.
\newblock Braintorrent: A peer-to-peer environment for decentralized federated
  learning.
\newblock {\em arXiv preprint arXiv:1905.06731}, 2019.

\bibitem{rubinstein2009learning}
Benjamin~IP Rubinstein, Peter~L Bartlett, Ling Huang, and Nina Taft.
\newblock Learning in a large function space: Privacy-preserving mechanisms for
  svm learning.
\newblock {\em arXiv preprint arXiv:0911.5708}, 2009.

\bibitem{mobilenetv2}
Mark Sandler, Andrew~G. Howard, Menglong Zhu, Andrey Zhmoginov, and
  Liang{-}Chieh Chen.
\newblock Inverted residuals and linear bottlenecks: Mobile networks for
  classification, detection and segmentation.
\newblock {\em CVPR}, abs/1801.04381, 2018.

\bibitem{shayan2018biscotti}
Muhammad Shayan, Clement Fung, Chris~JM Yoon, and Ivan Beschastnikh.
\newblock Biscotti: A ledger for private and secure peer-to-peer machine
  learning.
\newblock {\em arXiv preprint arXiv:1811.09904}, 2018.

\bibitem{shokri2015privacy}
Reza Shokri and Vitaly Shmatikov.
\newblock Privacy-preserving deep learning.
\newblock In {\em Proceedings of the 22nd ACM SIGSAC conference on computer and
  communications security}, pages 1310--1321. ACM, 2015.

\bibitem{vgg}
K.~Simonyan and A.~Zisserman.
\newblock Very deep convolutional networks for large-scale image recognition.
\newblock {\em International Conference on Learning Representations}, 2015.

\bibitem{googlenet}
Christian Szegedy, Wei Liu, Yangqing Jia, Pierre Sermanet, Scott~E. Reed,
  Dragomir Anguelov, Dumitru Erhan, Vincent Vanhoucke, and Andrew Rabinovich.
\newblock Going deeper with convolutions.
\newblock {\em CoRR}, abs/1409.4842, 2014.

\bibitem{efficientnet}
Mingxing Tan and Quoc~V. Le.
\newblock Efficientnet: Rethinking model scaling for convolutional neural
  networks.
\newblock {\em ICML}, abs/1905.11946, 2019.

\bibitem{truex2018hybrid}
Stacey Truex, Nathalie Baracaldo, Ali Anwar, Thomas Steinke, Heiko Ludwig, and
  Rui Zhang.
\newblock A hybrid approach to privacy-preserving federated learning.
\newblock {\em arXiv preprint arXiv:1812.03224}, 2018.

\bibitem{wang2019efficient}
Lingxiao Wang, Bargav Jayaraman, David Evans, and Quanquan Gu.
\newblock Efficient privacy-preserving nonconvex optimization.
\newblock {\em arXiv preprint arXiv:1910.13659}, 2019.

\bibitem{wang2018subsampled}
Yu-Xiang Wang, Borja Balle, and Shiva Kasiviswanathan.
\newblock Subsampled r{\'e}nyi differential privacy and analytical moments
  accountant.
\newblock {\em arXiv preprint arXiv:1808.00087}, 2018.

\bibitem{resnext}
Saining Xie, Ross~B. Girshick, Piotr Doll{\'{a}}r, Zhuowen Tu, and Kaiming He.
\newblock Aggregated residual transformations for deep neural networks.
\newblock {\em CVPR}, abs/1611.05431, 2017.

\bibitem{you2017imagenet}
Yang You, Zhao Zhang, James Demmel, Kurt Keutzer, and Cho-Jui Hsieh.
\newblock Imagenet training in 24 minutes.
\newblock {\em arXiv preprint arXiv:1709.05011}, 2017.

\bibitem{icassp19}
Wei Zhang, Xiaodong Cui, Ulrich Finkler, Brian Kingsbury, George Saon, David
  Kung, and Michael Picheny.
\newblock Distributed deep learning strategies for automatic speech
  recognition.
\newblock In {\em ICASSP'2019}, May 2019.

\bibitem{interspeech19}
Wei Zhang, Xiaodong Cui, Ulrich Finkler, George Saon, Abdullah Kayi, Alper
  Buyuktosunoglu, Brian Kingsbury, David Kung, and Michael Picheny.
\newblock A highly efficient distributed deep learning system for automatic
  speech recognition.
\newblock In {\em INTERSPEECH'2019}, Sept 2019.

\bibitem{icassp20}
Wei Zhang, Xiaodong Cui, Abdullah Kayi, Mingrui Liu, Ulrich Finkler, Brian
  Kingsbury, George Saon, Youssef Mroueh, Alper Buyuktosunoglu, Payel Das,
  David Kung, and Michael Picheny.
\newblock Improving efficiency in large-scale decentralized distributed
  training.
\newblock In {\em ICASSP'2020}, May 2020.

\bibitem{zhang2015staleness}
Wei Zhang, Suyog Gupta, Xiangru Lian, and Ji~Liu.
\newblock Staleness-aware async-sgd for distributed deep learning.
\newblock In {\em Proceedings of the Twenty-Fifth International Joint
  Conference on Artificial Intelligence, {IJCAI} 2016, New York, NY, USA, 9-15
  July 2016}, pages 2350--2356, 2016.

\bibitem{zhang2016icdm}
Wei Zhang, Suyog Gupta, and Fei Wang.
\newblock Model accuracy and runtime tradeoff in distributed deep learning: A
  systematic study.
\newblock In {\em IEEE International Conference on Data Mining}, 2016.

\bibitem{shufflenet}
Xiangyu Zhang, Xinyu Zhou, Mengxiao Lin, and Jian Sun.
\newblock Shufflenet: An extremely efficient convolutional neural network for
  mobile devices.
\newblock {\em CVPR}, abs/1707.01083, 2018.

\end{thebibliography}
}

\clearpage 
\appendix
\section{System Design and Implementation}
\label{sec:design}
Most (if not all) of the state-of-the-art DL models are trained in the Synchronous SGD fashion, as it is considered to have the most stable convergence behavior ~\cite{facebook-1hr, revisit-sync-sgd,zhang2016icdm} and often yields the best model accuracy. No existing FL training method has been proven to outperform \sync for the final model accuracy. 
We implement Synchronous SGD (\sync) as the "golden" baseline to examine if \adpdpsgd can achieve the best possbile model accuracy while maintaining DP protection. In \sync, we place an ``allreduce'' (sum) call after each learner's weight update in each iteration and then take the average of the weights across all the learners. An allreduce call is a reduction operation followed by a broadcast operation. A reduction operation is both commutative and associative (e.g., summation). The allreduce mechanism we chose is Nvidia NCCL ~\cite{nccl}, which is the state-of-the-art allreduce implementation on a GPU cluster.

To implement \adpsgd, we place all the learners (\emph{i.e.}, GPUs) on a communication ring. To avoid deadlock, we partition the communication ring into nonintersecting sets of senders and receivers and require that communication edges start from the sender sub-graph and end in the receiver sub-graph. To achieve better convergence behavior, we also adopt the communication randomization technique as proposed in ~\cite{icassp20}, in which sender randomly picks a receiver in each iteration. The paired-up sender and receiver exchange weights and update their weights as the average of the two. A global counter is maintained to record how many minibatches all the learners in the system collectively have processed. The training finishes when the global counter reaches the termination threshold. We implemented \adpsgd in C++ and MPI. 

In both \sync and \adpsgd, we insert random noise into gradients in each iteration to enable differential privacy protection.


\section{Additional Experiment Results}
\label{sec:add_exp}
\subsection{Software and Hardware}
We use PyTorch 1.1.0 as the underlying deep learning framework. We use the CUDA 10.1 compiler, the CUDA-aware OpenMPI 3.1.1, and g++ 4.8.5 compiler to build our communication library, which connects with PyTorch via a Python-C interface. We run our experiments on a 16-GPU 2-server cluster. Each server has 2 sockets and 9 cores per socket. Each core is an Intel Xeon E5-2697 2.3GHz processor. Each server is equipped with 1TB main memory and 8 V100 GPUs. Between servers are 100Gbit/s Ethernet connections. GPUs and CPUs are connected via PCIe Gen3 bus, which has a 16GB/s peak bandwidth in each direction per CPU socket.

\subsection{Dataset and Model}
We evaluate on two deep learning tasks: computer vision and speech recognition. For computer vision task, we evaluate on \cifar dataset~\cite{krizhevsky2009learning}, which comprises of a total of 60,000 RGB images of size 32 $\times$ 32  pixels partitioned into the training set (50,000 images) and the test set (10,000 images). We test \cifar with 9 representative convolutional neural network (CNN) models~\cite{pytorch-cifar}: (1) \shuf, a 50 layer instantiation of ShuffleNet architecture \cite{shufflenet}. (2) \mobilenetvtwo, a 19 layer instantiation of \cite{mobilenetv2} architecture that improves over MobileNet by introducing linear bottlenecks and inverted residual block.(3) \eff , with a compound coefficient 0 in the basic EfficientNet architecture~\cite{efficientnet}. (4) \mobilenet, a 28 layer instantiation of MobileNet architecture \cite{mobilenet}. (5) \googlenet, a 22 layer instantiation of Inception architecture \cite{googlenet}. (6) \resnext, a 29 layer instantiation of \cite{resnext} with bottlenecks width 64 and 2 sets of aggregated transformations. (7) \resnet, a 18 layer instantiation of ResNet architecture \cite{resnet}.(8) \senet, which stacks Squeeze-and-Excitation blocks~\cite{senet} on top of a ResNet-18 model. (9) \vgg, a 19 layer instantiation of VGG architecture \cite{vgg}.      The detailed model implementation refers to \cite{pytorch-cifar}.

Among these models, ShuffleNet, MobileNet(V2), EfficientNet represent the low memory footprint models that are widely used on mobile devices, where federated learnings is often used. The other models are standard CNN models that aim for high accuracy.

For speech recognition task, we evaluate on \swb dataset. The training dataset is 30GB large and contains roughly 4.2 million samples. The test dataset is 564MB large and contains roughly 80,000 samples. The acoustic model is a long short-term memory (LSTM) model with 6 bi-directional layers. Each layer contains 1,024 cells (512 cells in each direction). On top of the LSTM layers, there is a linear projection layer with 256 hidden units, followed by a softmax output layer with 32,000 (i.e. 32,000 classes) units corresponding to context-dependent HMM states. The LSTM is unrolled with 21 frames and trained with non-overlapping feature subsequences of that length.  The feature input is a fusion of FMLLR (40-dim), i-Vector (100-dim), and logmel with its delta and double delta (40-dim $\times$3). This model contains over 43 million parameters and is about 165MB large. 

\Cref{tab:model_size_training_time} summarizes the model size and training time for both tasks. Training time is measured on running single-GPU and accounts the time for collecting training statistics (e.g., L2-Norm) and noise injection operations.

\begin{table*}[htbp]
	\centering
	\begin{tabular}{|c|c|c|}
		\hline
		Model/Dataset               & Model Size (MB) & Training Time (hr) \\ \hline
		\shuf/C*     & 4.82            & 4.16                    \\ \hline
		\mobilenetvtwo/C    & 8.76            & 4.63                    \\ \hline
		\eff/C & 11.11           & 5.53                    \\ \hline
		\mobilenet/C      & 12.27           & 3.55                    \\ \hline
		\googlenet/C      & 23.53           & 8.43                    \\ \hline
		\resnext/C        & 34.82           & 7.08                    \\ \hline
		\resnet/C         & 42.63           & 5.56                    \\
		\hline
		\senet/C        & 42.95           & 5.80                    \\
		\hline
		\vgg/C            & 76.45           & 7.42                    \\ \hline
		LSTM/S**           & 164.62          & 102.41                  \\ \hline
	\end{tabular}
	\caption{Model size and training time. Training time is measured on running on 1 V100 GPU, which accounts additional operations for collecting training statistics (e.g., L2-Norm) and noise injection operations. *C stands for \cifar, **S stands for \swb.}
	\label{tab:model_size_training_time}
\end{table*}

\begin{figure}
	\setlength{\belowcaptionskip}{-0.5cm}
	\centering
	{\includegraphics[width=1\columnwidth]{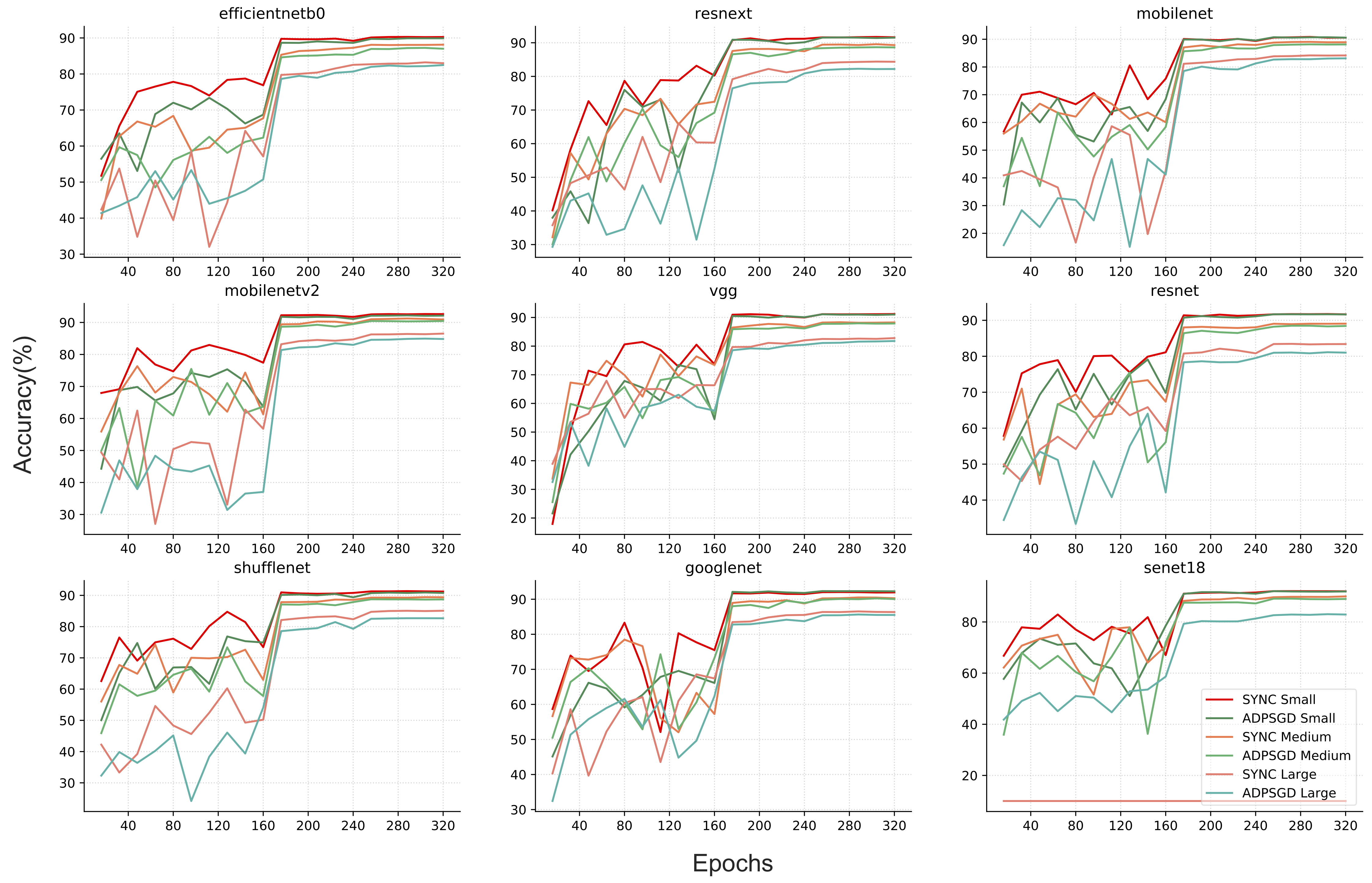}
	}
	\caption{\cifar convergence comparison between \sync and \adpsgd under various levels of noise injection (i.e., differential privacy budget). \sync and \adpsgd achieve similar level of utilities (i.e. test accuracy). }
\end{figure}

\begin{figure}
	{\includegraphics[width=1\columnwidth]{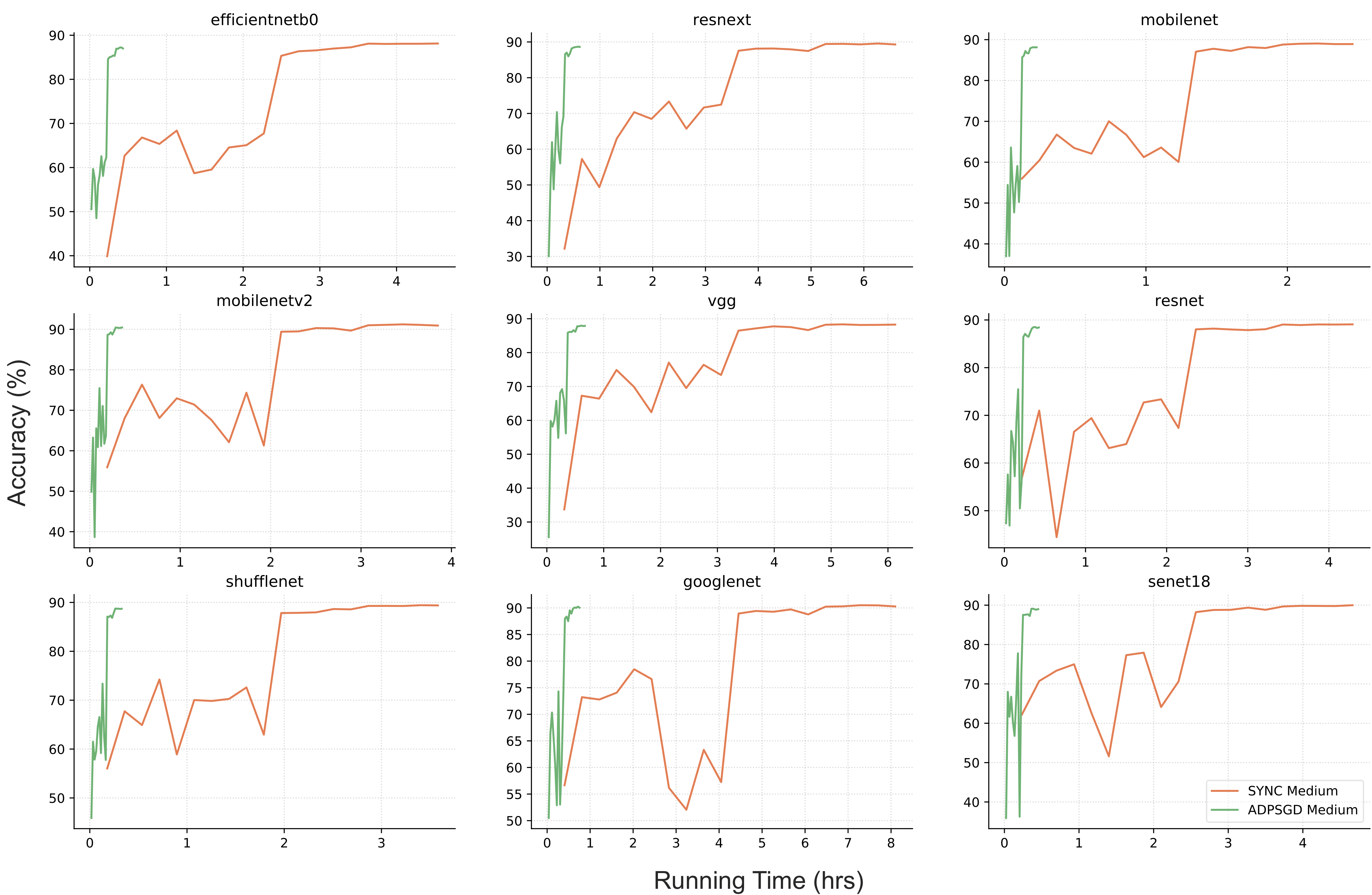}
	}
	\caption{\cifar convergence when one learner is slowed down by 10X in each iteration with medium level of noise injection. 
		\adpsgd runs significantly faster than \sync due to its asynchronous nature. }
	\label{fig:cifar-10x-slow}
\end{figure}

\begin{figure}[h]
	\setlength{\belowcaptionskip}{-0.5cm}
	{\includegraphics[width=1\columnwidth]{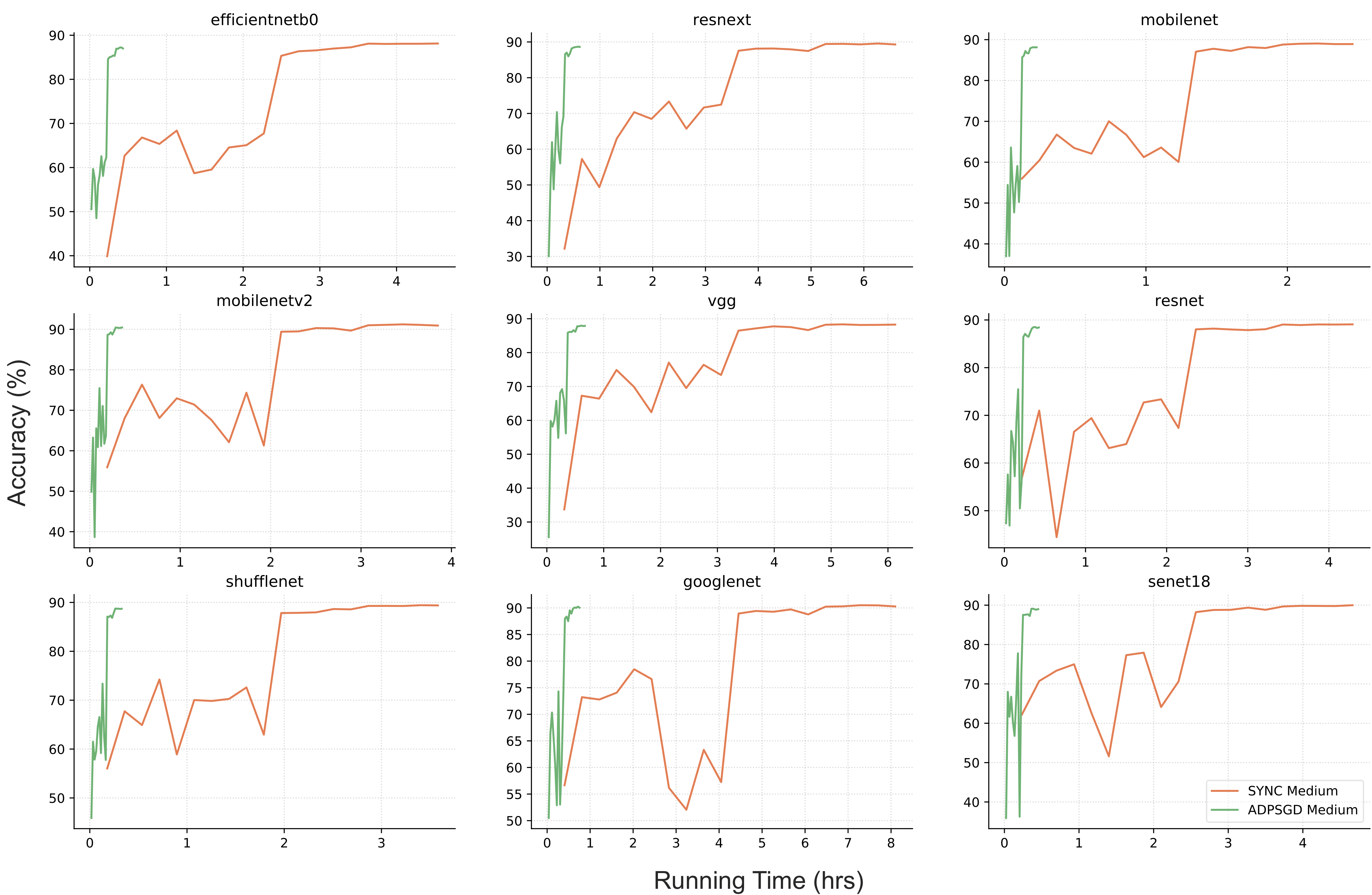}
	}
	\caption{\cifar convergence when one learner is slowed down by 10X in each iteration with medium level of noise injection (we omit displaying other levels of noise for the sake of brevity). \adpsgd runs significantly faster than \sync due to its asynchronous nature. }
	\label{fig:cifar-10x-slow}
\end{figure}
\Cref{fig:cifar-10x-slow} depicts the convergence comparison between \sync and \adpsgd on \cifar when 1 learner is slowed down by 10X.

\Cref{fig:swb-convergence-slow} depicts the convergence comparison between \sync and \adpsgd on \swb task for Case I and case II.

\begin{figure}
	\centering
	\begin{minipage}{0.49\textwidth}
		\centering
		\includegraphics[width=1\textwidth]{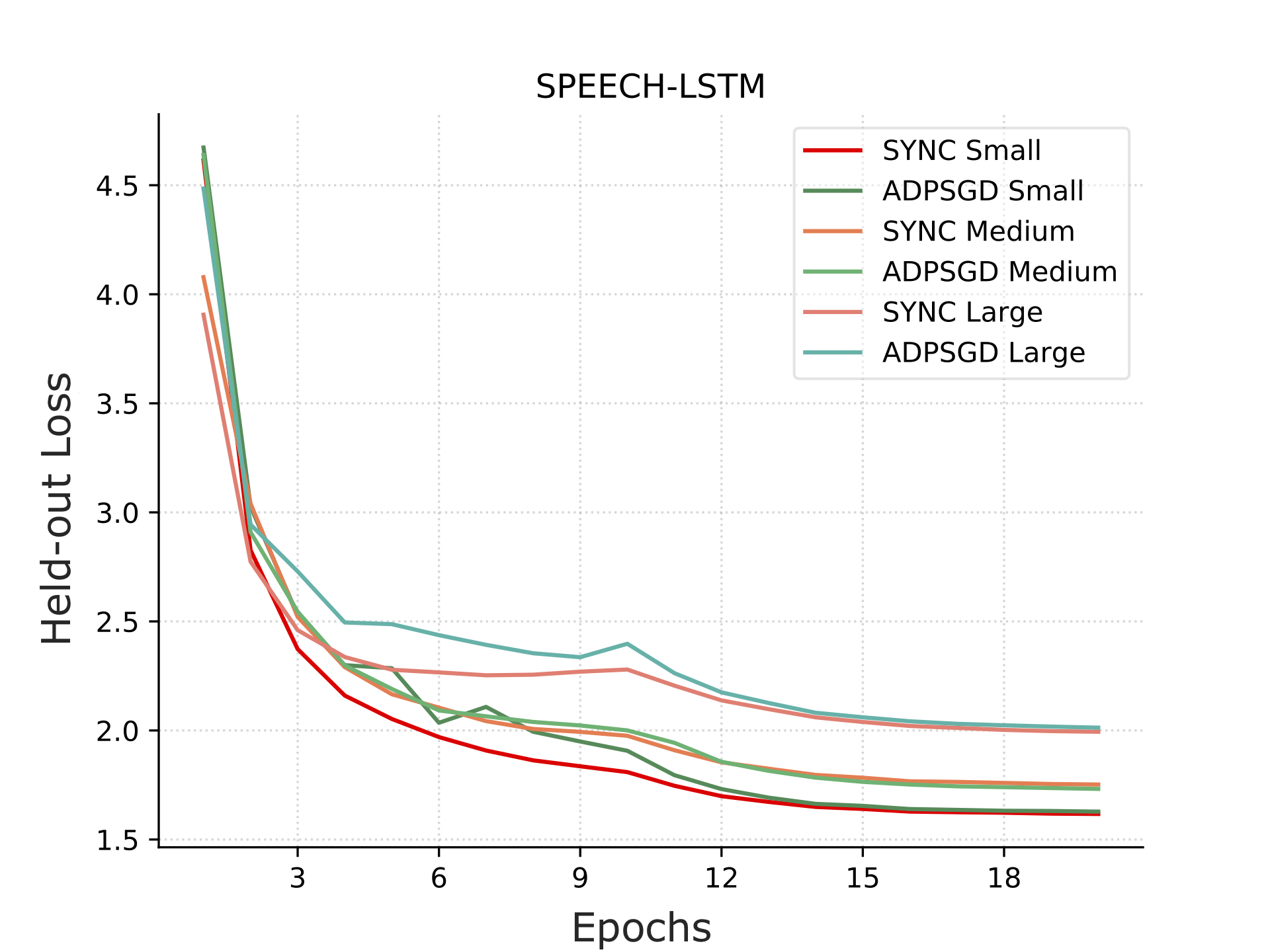} 
		\caption{\swb convergence comparison between \sync and \adpsgd under various levels of noise injection (i.e., differential privacy budget). \sync and \adpsgd achieve similar level of utilities (i.e. held-out loss). }
		\label{fig:swb-convergence}
	\end{minipage}\hfill
	\begin{minipage}{0.49\textwidth}
		\centering
		\includegraphics[width=1\textwidth]{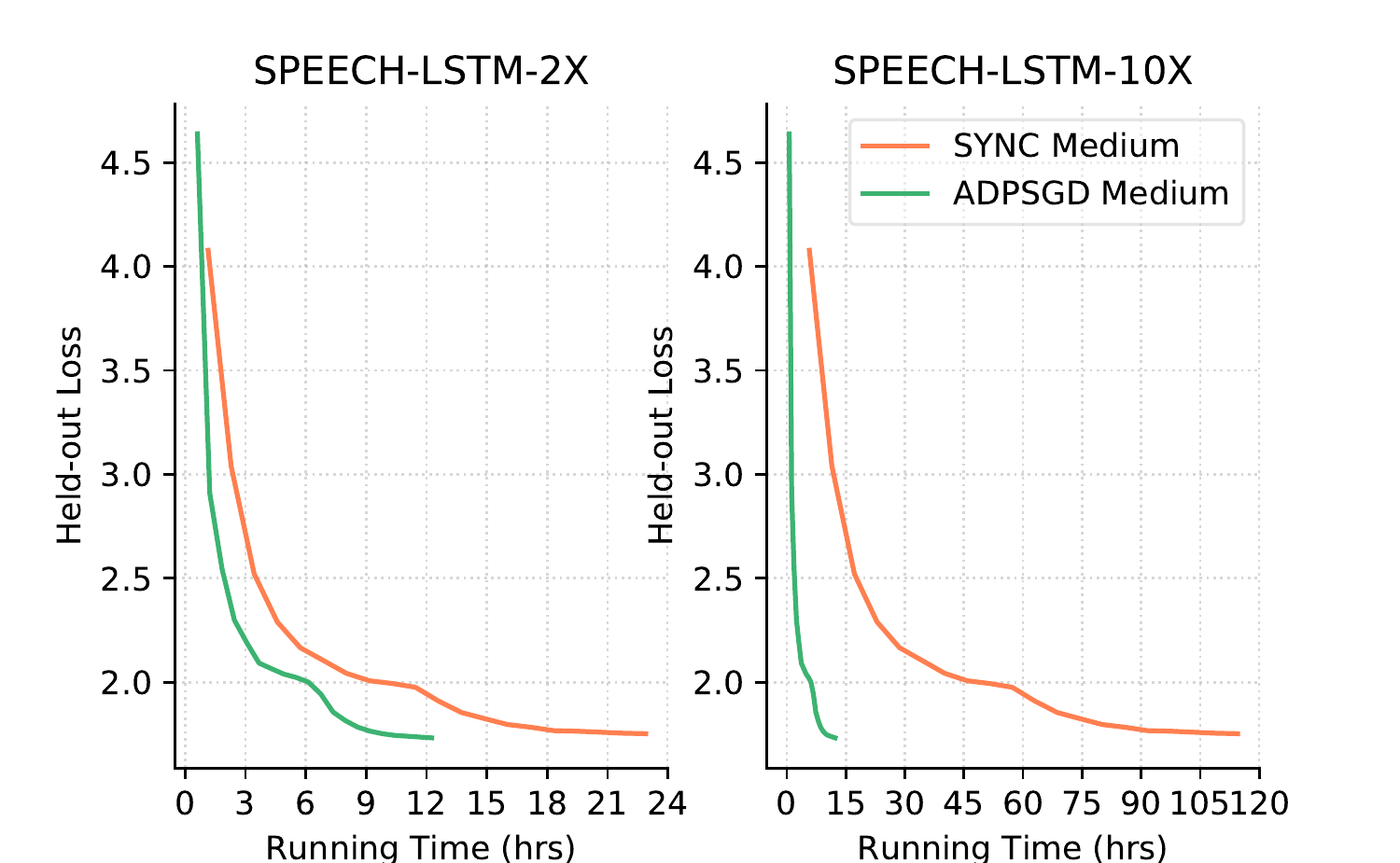} 
		\caption{\swb convergence when one random learner is slowed down by 2X (left) and one learner is slowed down by 10X (right), with medium level noise injection (we omit displaying other levels of noise for the sake of brevity). \adpsgd runs significantly faster than \sync due to its asynchronous nature.}
		\label{fig:swb-convergence-slow}
	\end{minipage}
\end{figure}
\section{Proofs of Theorem~\ref{the:utility}}
\begin{proof}
	Recall that we have the following update rule:
	\begin{align}
		&[\mathbf{w}_1^{t+1/2},\mathbf{w}_2^{t+1/2},...,\mathbf{w}_K^{t+1/2}]\leftarrow[\mathbf{w}_1^t,\mathbf{w}_2^t,...,\mathbf{w}_K^t]\mathbf{A}^t;\notag\\
		&\mathbf{w}_{k^t}^{t+1} \leftarrow \mathbf{w}_{k^t}^{t+1/2}-\eta \left({g}^t(\hat{\mathbf{w}}_{k^t}^t;\xi_{k^t}^t)+\mathbf{n}\right),
	\end{align}
	where $k^t$ are drawn from $[K]$, and $\mathbf{n}\sim \mathcal{N}(0,\sigma^2\mathbf{I})$. Let's denote $\theta^{t+1}$ as the average of all local models at $\{t+1\}$-th iteration, \emph{i.e.}, $\theta^{t+1}=\frac{1}{K}\sum_{k=1}^K \mathbf{w}_k^{t+1}$. Since $F$ is $L$-smooth, we have
	\begin{equation*}
		\begin{aligned}
			F(\theta^{t+1})
			=&F\left(\frac{\mathbf{W}^{t+1}\mathbf{1}_K}{K}\right)
			=F\left(\frac{\mathbf{W}^{t}\mathbf{A}^t\mathbf{1}_K}{K}-\frac{\eta}{K} \left( g_c(\hat{\mathbf{w}}_{k^t}^t,\xi_{k^t}^t)+\mathbf{n}\right)\right)\\
			=&F\left(\frac{\mathbf{W}^{t}\mathbf{1}_K}{K}-\frac{\eta}{K} \left( g_c(\hat{\mathbf{w}}_{k^t}^t,\xi_{k^t}^t)+\mathbf{n}\right)\right)
			=F\left(\theta^{t}-\frac{\eta}{K} \left( g_c(\hat{\mathbf{w}}_{k^t}^t,\xi_{k^t}^t)+\mathbf{n}\right)\right)\\
			\leq&F\left(\theta^{t}\right)-\frac{\eta}{K}\left\langle \triangledown F(\theta^t), \left( g_c(\hat{\mathbf{w}}_{k^t}^t,\xi_{k^t}^t)+\mathbf{n}\right)\right\rangle\\
			&\ \ \ +\frac{\eta^2L}{2K^2}\left(\left\|g_c(\hat{\mathbf{w}}_{k^t}^t,\xi_{k^t}^t)\right\|^2+\|\mathbf{n}\|^2+2\left\langle g_c(\hat{\mathbf{w}}_{k^t}^t,\xi_{k^t}^t),\mathbf{n} \right\rangle\right).\\
		\end{aligned}
	\end{equation*}
	
	Tanking expectation with respect to $k^t$ and $\mathbf{n}$ given $\theta^t$, we have
	\begin{align}\label{eq:exp}
		&\mathbb{E}F(\theta^{t+1})-\mathbb{E}F(\theta^t)\notag\\
		\leq&-\frac{\eta}{K}\mathbb{E}\left\langle \triangledown F(\theta^t), g(\hat{\mathbf{w}}_{k^t}^t,\xi_{k^t}^t)\right\rangle
		+\frac{\eta^2L}{2K^2}\left(\mathbb{E}\left\|g(\hat{\mathbf{w}}_{k^t}^t,\xi_{k^t}^t)\right\|^2+\mathbb{E}\|\mathbf{n}\|^2\right)\notag\\
		\leq&-\frac{\eta B}{K}\mathbb{E}\left\langle \triangledown F(\theta^t),  \sum_{k=1}^{K}p_k\triangledown F_k(\hat{\mathbf{w}}_k^t)\right\rangle
		+\frac{\eta^2L}{2K^2}\mathbb{E}\|g(\hat{\mathbf{w}}_{k^t}^t;\xi_{k^t}^t)\|^2+\frac{\eta^2Ld\sigma^2}{2K^2}\notag\\
		=&\frac{\eta B}{2K}\underbrace{\mathbb{E}\left\| \triangledown F(\theta^t)- \sum_{k=1}^Kp_k\triangledown F_k(\hat{\mathbf{w}}_k^t)\right\|^2}_{T_1} -\frac{\eta B}{2K}\mathbb{E}\left\|\sum_{k=1}^Kp_k\triangledown F_k(\hat{\mathbf{w}}_k^t)\right\|^2\notag\\
		&-\frac{\eta B}{2K}\mathbb{E}\left\| \triangledown F(\theta^t)\right\|^2+\frac{\eta^2L}{2K^2}\underbrace{\mathbb{E}\left\|g(\hat{\mathbf{w}}_{k^t}^t,\xi_{k^t}^t)\right\|^2}_{T_2} + \frac{\eta^2Ld\sigma^2}{2K^2}.
	\end{align}
	
	For $T_1$ we have
	\begin{align}\label{eq:T1}
		T_1=&\mathbb{E}\left\| \triangledown F(\theta^t)- \sum_{k=1}^Kp_k \triangledown F_k(\hat{\mathbf{w}}_k^t)\right\|^2\notag\\
		\leq&2\mathbb{E}\left\| \triangledown F(\theta^t)- \triangledown F(\hat{\theta}^t)\right\|^2+2\mathbb{E}\left\| \triangledown F(\hat{\theta}^t)- \sum_{k=1}^Kp_k \triangledown F_k(\hat{\mathbf{w}}_k^t)\right\|^2\notag\\
		=&2\mathbb{E}\left\| \triangledown F(\theta^t)- \triangledown F(\hat{\theta}^t)\right\|^2+2\mathbb{E}\left\|\sum_{k=1}^Kp_k\left( \triangledown F_k(\hat{\theta}^t)-  \triangledown F_k(\hat{\mathbf{w}}_k^t)\right)\right\|^2\notag\\
		\leq&2\mathbb{E}\left\| \triangledown F(\theta^t)- \triangledown F(\hat{\theta}^t)\right\|^2+2\mathbb{E}\sum_{k=1}^Kp_k\left\| \triangledown F_k(\hat{\theta}^t)-  \triangledown F_k(\hat{\mathbf{w}}_k^t)\right\|^2\notag\\
		\leq& 2L^2\mathbb{E}\left\|\theta^t-\hat{\theta}^t\right\|^2+2L^2\mathbb{E}\sum_{k=1}^Kp_k\left\|\hat{\theta}^t
		-\hat{\mathbf{w}}_k^t\right\|^2\notag\\
		=&2L^2\mathbb{E}\left\|\sum_{t'=1}^{\tau_t}\frac{\eta}{K} g(\hat{\mathbf{w}}_{k^{t-t'}}^{t-t'},\xi_{k^{t-t'}}^{t-t'})\right\|^2+2L^2\mathbb{E}\sum_{k=1}^Kp_k\left\|\hat{\theta}^t   -\hat{\mathbf{w}}_k^t\right\|^2\notag\\
		\leq&\frac{2\eta^2L^2\tau_t}{K^2}\sum_{t'=1}^{\tau_t}\mathbb{E}\left\|g(\hat{\mathbf{w}}_{k^{t-t'}}^{t-t'},\xi_{k^{t-t'}}^{t-t'})\right\|^2+2L^2\mathbb{E}\sum_{k=1}^Kp_k\left\|\hat{\theta}^t   -\hat{\mathbf{w}}_k^t\right\|^2.
	\end{align}
	
	For $T_2$, we have
	\begin{align}\label{eq:T2}
		T_2 = &\mathbb{E}\left\|g(\hat{\mathbf{w}}_{k^t}^t,\xi_{k^t}^t)\right\|^2 =
		\sum_{k=1}^Kp_k\mathbb{E}\left\|\sum_{j=1}^{B}\triangledown f_j(\hat{\mathbf{w}}_k^t,\xi_{k}^{t,j})\right\|^2\notag\\
		=&\sum_{k=1}^Kp_k\mathbb{E}\left\|\sum_{j=1}^B\left(\triangledown f(\hat{\mathbf{w}}_k^t,\xi_k^{t,j})-\triangledown F_k(\hat{\mathbf{w}}_k^t)\right)\right\|^2 + \sum_{k=1}^Kp_k\mathbb{E}\left\|B\triangledown F_k(\hat{\mathbf{w}}_k^t)\right\|^2\notag\\
		\leq&B\varsigma^2+B^2\sum_{k=1}^Kp_k\mathbb{E}\left\|\triangledown F_k(\hat{\mathbf{w}}_k^t)\right\|^2.
	\end{align}
	
	According to Lemma 5 in~\cite{lian2018asynchronous}, we have
	\begin{equation}\label{eq:lem5}
	\sum_{k=1}^Kp_k\mathbb{E}\left\|\triangledown F_k(\hat{\mathbf{w}}_k^t)\right\|^2
	\leq 12L^2\mathbb{E}\sum_{k=1}^Kp_k^2\left\|\hat{\theta}^t   -\hat{\mathbf{w}}_k^t\right\|^2+6\upsilon^2+2\mathbb{E}\left\|\sum_{k=1}^Kp_k\triangledown F_k(\hat{\mathbf{w}}_k^t)\right\|^2.\notag
	\end{equation}
	
	Plugging (\ref{eq:T1}), (\ref{eq:T2}) and (\ref{eq:lem5}) into (\ref{eq:exp}), we can obtain
	\begin{align}
		&\mathbb{E}F(\theta^{t+1})-\mathbb{E}F(\theta^t)\notag\\
		\stackrel{(\ref{eq:T1})}\leq&\frac{\eta B}{K}\left(\frac{\eta^2L^2\tau_t}{K^2}\sum_{t'=1}^{\tau_t}\mathbb{E}\left\|g(\hat{\mathbf{w}}_{k^{t-t'}}^{t-t'},\xi_{k^{t-t'}}^{t-t'})\right\|^2+L^2\mathbb{E}\sum_{k=1}^Kp_k\left\|\hat{\theta}^t   -\hat{\mathbf{w}}_k^t\right\|^2\right)\notag\\
		& -\frac{\eta B}{2K}\mathbb{E}\left\|\sum_{k=1}^Kp_k \triangledown F_k(\hat{\mathbf{w}}_k^t)\right\|^2-\frac{\eta B}{2K}\mathbb{E}\left\| \triangledown F(\theta^t)\right\|^2\notag\\
		&+\frac{\eta^2L}{2K^2}\mathbb{E}\left\|g(\hat{\mathbf{w}}_{k^t}^t,\xi_{k^t}^t)\right\|^2+\frac{\eta^2Ld\sigma^2}{2K^2}\notag\\
		\stackrel{(\ref{eq:T2})}\leq&\frac{\eta^3L^2B^3\tau_t}{K^3}\sum_{t'=1}^{\tau_t}\sum_{k=1}^Kp_k\mathbb{E}\left\|\triangledown F_k(\hat{\mathbf{w}}_k^{t-t'})\right\|^2 +
		\frac{\eta L^2B}{K}\mathbb{E}\sum_{k=1}^Kp_k\left\|\hat{\theta}^t   -\hat{\mathbf{w}}_k^t\right\|^2\notag\\
		&-\frac{\eta B}{2K}\mathbb{E}\left\|\sum_{k=1}^Kp_k \triangledown F_k(\hat{\mathbf{w}}_k^t)\right\|^2-\frac{\eta B}{2K}\mathbb{E}\left\| \triangledown F(\theta^t)\right\|^2
		+\frac{\eta^2LB^2}{2K^2}\sum_{k=1}^Kp_k\mathbb{E}\left\|\triangledown F_k(\hat{\mathbf{w}}_k^t)\right\|^2\notag\\
		&+\frac{\eta^2L(\varsigma^2B+d\sigma^2)}{2K^2}+\frac{\eta^3L^2B^2\varsigma^2\tau_t^2}{K^3}\notag\\\notag
		\stackrel{(\ref{eq:lem5})}\leq&\frac{2\eta^3L^2B^3\tau}{K^3}\sum_{t'=1}^{\tau_t}
		\left(6L^2\mathbb{E}\sum_{k=1}^Kp_k\left\|\hat{\theta}^{t-t'}   -\hat{\mathbf{w}}_k^{t-t'}\right\|^2+\mathbb{E}\left\|\sum_{k=1}^Kp_k\triangledown F_k(\hat{\mathbf{w}}_k^{t-t'})\right\|^2\right)\notag\\
		&-\frac{\eta B}{2K}\mathbb{E}\left\| \triangledown F(\theta^t)\right\|^2+\left(\frac{\eta L^2B}{K}+\frac{6\eta^2L^3B^2}{K^2}\right)\mathbb{E}\sum_{k=1}^Kp_k\left\|\hat{\theta}^t   -\hat{\mathbf{w}}_k^t\right\|^2
		\notag\\
		& -\left(\frac{\eta B}{2K}-\frac{\eta^2LB^2}{K^2}\right)\mathbb{E}\left\|\sum_{k=1}^Kp_k \triangledown F_k(\hat{\mathbf{w}}_k^t)\right\|^2\notag\\
		&+\frac{\eta^2L(\varsigma^2B+6\upsilon^2B^2+d\sigma^2)}{2K^2}+\frac{\eta^3L^2B^2\tau^2(\varsigma^2+6\upsilon^2B)}{K^3}\notag.
	\end{align}
	
	Summing over $t=0,...,T-1$ and $\forall \tau_t\leq \tau$, we have
	\begin{align}
		&\mathbb{E}F(\theta^{t+1})-\mathbb{E}F(\theta^t)\notag\\
		\leq&\frac{2\eta^3B^3L^2\tau^2}{K^3}\sum_{t=0}^{T-1}
		\left(6L^2\mathbb{E}\sum_{k=1}^Kp_k\left\|\hat{\theta}^{t}   -\hat{\mathbf{w}}_k^{t}\right\|^2+\mathbb{E}\left\|\sum_{k=1}^Kp_k\triangledown F_k(\hat{\mathbf{w}}_k^{t})\right\|^2\right)\notag\\
		& -\left(\frac{\eta B}{2K}-\frac{\eta^2B^2L}{K^2}\right)\sum_{t=0}^{T-1}\mathbb{E}\left\|\sum_{k=1}^Kp_k \triangledown F_k(\hat{\mathbf{w}}_k^t)\right\|^2-\frac{\eta B}{2K}\sum_{t=0}^{T-1}\mathbb{E}\left\| \triangledown F(\theta^t)\right\|^2\notag\\
		&+\left(\frac{\eta BL^2}{K} +\frac{6\eta^2B^2L^3}{K^2}\right)\sum_{t=0}^{T-1}\mathbb{E}\sum_{k=1}^Kp_k\left\|\hat{\theta}^t   -\hat{\mathbf{w}}_k^t\right\|^2
		\notag\\
		&+\frac{\eta^3BL^2\tau^2(\varsigma^2B+6\upsilon^2B^2+d\sigma^2)T}{K^3}+\frac{\eta^2L(\varsigma^2B+6\upsilon^2B^2+d\sigma^2)T}{2K^2}\notag\\
		=&\left(\frac{\eta BL^2}{K} +\frac{6\eta^2B^2L^3}{K^2}+\frac{12\eta^3B^3L^4\tau^2}{K^3}\right)\underbrace{\sum_{t=0}^{T-1}\mathbb{E}\sum_{k=1}^Kp_k\left\|\hat{\theta}^t   -\hat{\mathbf{w}}_k^t\right\|^2}_{T_3}
		\notag\\
		& -\left(\frac{\eta B}{2K}-\frac{\eta^2B^2L}{K^2}-\frac{2\eta^3B^3L^2\tau^2}{K^3}\right)\sum_{t=0}^{T-1}\mathbb{E}\left\|\sum_{k=1}^Kp_k \triangledown F_k(\hat{\mathbf{w}}_k^t)\right\|^2-\frac{\eta B}{2K}\sum_{t=0}^{T-1}\mathbb{E}\left\| \triangledown F(\theta^t)\right\|^2\notag\\
		&+\frac{\eta^3BL^2\tau^2(\varsigma^2B+6\upsilon^2B^2+d\sigma^2)T}{K^3}+\frac{\eta^2L(\varsigma^2B+6\upsilon^2B^2+d\sigma^2)T}{2K^2}
	\end{align}  
	
	We can use Lemma~\ref{lem:T3} to bound the term $T_3$ and we use similar notations as in~\cite{lian2018asynchronous} for simpler notation. By arranging the terms, we have
	\begin{align}
		\mathbb{E}F(\theta^T)
		\leq&\mathbb{E}F(\theta^0)-C_2\sum_{t=0}^{T-1}\mathbb{E}\left\|\sum_{k=1}^Kp_k \triangledown F_k(\hat{\mathbf{w}}_k^t)\right\|^2-\frac{\eta B}{2K}\sum_{t=0}^{T-1}\mathbb{E}\left\| \triangledown F(\theta^t)\right\|^2\notag\\
		&+C_3\frac{\eta^2LT}{K^2}(\varsigma^2B+6\upsilon^2B^2+d\sigma^2).
	\end{align}  
	
	Then, while $C_3\leq 1$ and $C_2\geq 0$ we complete the  proof of Theorem~\ref{the:utility}.
	\begin{align}
		\frac{\sum_{t=0}^{T-1}\mathbb{E}\left\| \triangledown F(\theta^t)\right\|^2}{T}
		\leq &\frac{2\left(\mathbb{E}F(\theta^0)-\mathbb{E}F(\theta^T)\right)}{\eta TB/K}+\frac{2\eta L}{BK}(\varsigma^2B+6\upsilon^2B^2+d\sigma^2)\notag\\
		\leq &\frac{2\left(\mathbb{E}F(\mathbf{w}^0)-\mathbb{E}F^{*}\right)}{\eta TB/K}+\frac{2\eta L}{BK}(\varsigma^2B+6\upsilon^2B^2+d\sigma^2).\notag
	\end{align} 
	We complete the proof.
\end{proof}

\section{Proofs of Lemma~\ref{lem:T3}}
\begin{lemma}\label{lem:T3}
	While $C_1>0$ and $\forall T\geq 1$, we have
	\begin{align}
		\frac{\sum_{t=0}^{T-1}\mathbb{E}\sum_{k=1}^Kp_k\left\|\hat{\theta}^t -\hat{\mathbf{w}}_k^t\right\|^2}{T}
		\leq&\frac{4\eta^2B^2}{TC_1}\left(\tau\frac{K-1}{K}+\bar{\rho}\right)\sum_{t=0}^{T-1}\mathbb{E}\left\|\sum_{k=1}^Kp_k\triangledown F_k(\hat{\mathbf{w}}_k^{t})\right\|^2\notag\\
		&\quad\quad\quad +\frac{2\eta^2(\varsigma^2B+6\upsilon^2B^2+d\sigma^2)\bar{\rho}}{C_1}
	\end{align}
\end{lemma}
\begin{proof}
	First, we have
	\begin{align}
		&\mathbb{E}\left\|g(\hat{\mathbf{w}}_{k^t}^t,\xi_{k^t}^t)+\mathbf{n}-B\triangledown F_{k^t}(\mathbf{w}_{k^t}^t)\right\|^2\notag\\
		= &\mathbb{E}\left\|g(\hat{\mathbf{w}}_{k^t}^t,\xi_{k^t}^t)-B\triangledown F_{k^t}(\mathbf{w}_{k^t}^t)\right\|^2+\mathbb{E}\|\mathbf{n}\|^2
		\leq B\varsigma^2+d\sigma^2.
	\end{align}
	According to our updating rule and Lemma 6 in~\cite{lian2018asynchronous}, we can obtain
	\begin{align}
		&\mathbb{E}\left\|\theta^{t+1}- \mathbf{w}_k^{t+1}\right\|^2\notag\\
		=&\mathbb{E}\left\|\theta^{t}-\frac{\eta}{K} (g(\hat{\mathbf{w}}_{k^t}^t,\xi_{k^t}^t)+\mathbf{n})-\left(\mathbf{W}^t\mathbf{A}^t\mathbf{e}_k-\eta\tilde{g}(\hat{\mathbf{W}}^t,\xi_{k^t}^t)\mathbf{e}_k \right)\right\|^2\notag\\
		\leq & 2\eta^2B^2 \frac{K-1}{K}\mathbb{E}\sum_{j=0}^t\left(12L^2\sum_{k=1}^Kp_k\left\|\hat{\theta}^j -\hat{\mathbf{w}}_k^j\right\|^2+2\mathbb{E}\left\|\sum_{k=1}^Kp_k\triangledown F_k(\hat{\mathbf{w}}_k^j)\right\|^2 \right)\notag\\
		&\times(\rho^{t-j}+2(t-j)\rho^{\frac{t-j}{2}})
		+2\eta^2(\varsigma^2B+6\upsilon^2B^2+d\sigma^2)\bar{\rho}.
	\end{align}
	where $\bar{\rho}=\frac{K-1}{K}\left(\frac{1}{1-\rho}+\frac{2\sqrt{\rho}}{(1-\sqrt{\rho})^2}\right)$.
	
	Noting that $\hat{\mathbf{W}}^t=\mathbf{W}^{t-\tau_t}$, then we have
	\begin{align}
		&\mathbb{E}\sum_{k=1}^Kp_k\left\|\hat{\theta}^{t}  - \hat{\mathbf{w}}_k^{t}\right\|^2=\sum_{k=1}^Kp_k\mathbb{E}\left\|\hat{\theta}^{t}  - \hat{\mathbf{w}}_k^{t}\right\|^2\notag\\
		\leq & 2\eta^2B^2 \frac{K-1}{K}\mathbb{E}\sum_{j=0}^{t-\tau_t-1}\left(12L^2\sum_{k=1}^Kp_k\left\|\hat{\theta}^j   -\hat{\mathbf{w}}_k^j\right\|^2+2\mathbb{E}\left\|\sum_{k=1}^Kp_k\triangledown F_k(\hat{\mathbf{w}}_k^j)\right\|^2 \right)\notag\\
		&\times(\rho^{t-\tau_t-1-j}+2(t-\tau_t-1-j)\rho^{\frac{t-\tau_t-1-j}{2}})
		+2\eta^2(\varsigma^2B+6\upsilon^2B^2+d\sigma^2)\bar{\rho}.\notag
	\end{align}
	According to Lemma 7 in~\cite{lian2018asynchronous}, we can obtain
	\begin{align}
		&\frac{\sum_{t=0}^{T-1}\sum_{k=1}^Kp_k\mathbb{E}\left\|\hat{\theta}^{t}  - \hat{\mathbf{w}}_k^{t}\right\|^2}{T}\notag\\
		\leq & \frac{2\eta^2B^2}{T} \frac{K-1}{K}\sum_{t=0}^{T-1}\sum_{j=0}^{t-\tau_t-1}\left(12L^2\sum_{k=1}^Kp_k\left\|\hat{\theta}^j   -\hat{\mathbf{w}}_k^j\right\|^2+2\mathbb{E}\left\|\sum_{k=1}^Kp_k\triangledown F_k(\hat{\mathbf{w}}_k^j)\right\|^2 \right)\notag\\
		&\times(\rho^{t-\tau_t-1-j}+2(t-\tau_t-1-j)\rho^{\frac{t-\tau_t-1-j}{2}})
		+2\eta^2(\varsigma^2B+6\upsilon^2B^2+d\sigma^2)\bar{\rho}\notag\\
		\leq&\frac{4\eta^2B^2}{T} \left(\tau\frac{K-1}{K}+\bar{\rho}\right)\sum_{t=0}^{T-1}\left(\mathbb{E}\left\|\sum_{k=1}^Kp_k\triangledown F_k(\hat{\mathbf{w}}_k^t)\right\|^2 \right)\notag\\
		&+\frac{24\eta^2B^2L^2}{T}\left(\tau\frac{K-1}{K}+\bar{\rho}\right)\sum_{t=0}^{T-1}\mathbb{E}\sum_{k=1}^Kp_k\left\|\hat{\theta}^t -\hat{\mathbf{w}}_k^t\right\|^2+2\eta^2(\varsigma^2B+6\upsilon^2B^2+d\sigma^2)\bar{\rho}.
	\end{align}
	
	By rearranging the terms we obtain
	\begin{align}
		&\underbrace{\left(1-24\eta^2B^2L^2\left(\tau\frac{K-1}{K}+\bar{\rho}\right)\right)}_{C_1}\frac{\sum_{t=0}^{T-1}\mathbb{E}\sum_{k=1}^Kp_k\left\|\hat{\theta}^t   -\hat{\mathbf{w}}_k^t\right\|^2}{T}\notag\\
		\leq&\frac{4\eta^2B^2}{T}\left(\tau\frac{K-1}{K}+\bar{\rho}\right)\sum_{t=0}^{T-1}\mathbb{E}\left\|\sum_{k=1}^Kp_k\triangledown F_k(\hat{\mathbf{w}}_k^{t})\right\|^2+2\eta^2(\varsigma^2B+6\upsilon^2B^2+d\sigma^2)\bar{\rho}.
	\end{align}
	We complete the proof.
\end{proof}

\section{Proofs of Theorem~\ref{the:privacy}}
\begin{deff}[$(\alpha,\epsilon)$-RDP~\cite{mironov2017renyi}]
	A randomized mechanism $\mathcal{M}:\mathcal{X}^n\to\mathcal{R}$ satisfies $\epsilon$-R\'{e}nyi differential privacy of order $\alpha\in(1,\infty)$, or $(\alpha,\epsilon)$-RDP for short,
	if for all $\mathbf{x}, \mathbf{x}'\in\mathcal{X}^n$ differing on a single entry, it holds that
	\begin{equation}
	D_{\alpha}(\mathcal{M}(\mathbf{x})\|\mathcal{M}(\mathbf{x}'))\triangleq
	\frac{1}{\alpha-1}\log\mathbb{E}_{x\sim\mathcal{M}(\mathbf{x}')}\left(\frac{\mathcal{M}(\mathbf{x})(x)}{\mathcal{M}(\mathbf{x}')(x)}\right)^{\alpha}\leq\epsilon.
	\end{equation}
\end{deff}

The new definition $(\alpha,\epsilon)$-RDP shares many important properties with the standard definition of differential privacy, while additionally allowing for a more rigorous analysis of composite heterogeneous mechanisms~\cite{mironov2017renyi}.

To achieve differential privacy, a stochastic component (typically by additional noise) is usually added to or removed from the locally trained model. Typically, Gaussian mechanism, one of the common choices, injects additive Gaussian noise to the query:
\begin{equation}
\mathcal{M} \triangleq f(\mathbf{x})+\mathcal{N}(0,\sigma^2\triangle_2^2(f)),
\end{equation}
where $\mathcal{N}(0,\sigma^2\triangle_2^2(f))$ is the Gaussian distribution noise with mean $0$ and standard deviation $\sigma\triangle_2(f)$ which depends on the privacy budget $\epsilon$ as well as the sensitivity of $f$. And the (global) sensitivity of a function $f$ is defined as:
\begin{deff}[$l_p$-Sensitivity]
	The $l_p$-sensitivity of a function $f$ is defined by
	\begin{equation}
	\triangle_p(f) = \max_{\mathbf{x},\mathbf{x}'}\|f(\mathbf{x})-f(\mathbf{x}')\|_p,
	\end{equation}
	where $\mathbf{x}$ and $\mathbf{x}'$ differ in only one entry.
\end{deff}

In order to prove Theorem~\ref{the:privacy}, we need the following Lemmas.

\begin{lemma}[Gaussian Mechanism ~\cite{mironov2017renyi,wang2018subsampled,wang2019efficient}]\label{lem:gau}
	Given a function $f: \mathcal{X}^n\rightarrow \mathcal{R}$, the Gaussian Mechanism $\mathcal{M} \triangleq f(\mathbf{x})+\mathcal{N}(0,\sigma^2\mathbf{I})$ satisfies $(\alpha, \alpha\triangle_2^2(f)/(2\sigma^2))$-RDP. In addition, if $\mathcal{M}$ is applied to a subset of samples using uniform sampling without replacement $\mathcal{S}_{\gamma}^{\textsf{wo}}$, then $\mathcal{M}^{\mathcal{S}_{\gamma}^{\textsf{wo}}}$ that applies $\mathcal{M}\circ\mathcal{S}_{\gamma}^{\textsf{wo}}$ obeys $(\alpha, 5\gamma^2\alpha\triangle_2^2(f)/\sigma^2)$-RDP when $\sigma^2/\triangle_2^2(f)\geq1.5$ and  $\alpha\leq\log(1/(\gamma(1+\sigma^2/\triangle_2^2(f))))$ with $\gamma$ denoting the subsample rate.
\end{lemma}

\begin{lemma}[Composition ~\cite{mironov2017renyi,wang2019efficient}]\label{lem:seq_com}
	Let $\mathcal{M}_i: \mathcal{X}^n\rightarrow \mathcal{R}_i$ be an $(\alpha,\epsilon_i)$-RDP mechanism for $i\in[k]$. If $\mathcal{M}_{[k]}: \mathcal{X}^n\rightarrow \prod_{i=1}^k \mathcal{R}_i$ is defined to be $\mathcal{M}_{[k]}(\mathbf{x})=(\mathcal{M}_1(\mathbf{x}),...,\mathcal{M}_k(\mathbf{x}))$, then $\mathcal{M}_{[k]}$ is $(\alpha,\sum_{i=1}^k\epsilon_i)$-RDP. In addition, the input of $\mathcal{M}_i$ can be based on the outputs of previous $(i-1)$ mechanisms.
\end{lemma}

\begin{lemma}[From RDP to $(\epsilon,\delta)$-DP ~\cite{mironov2017renyi}]\label{lem:toDP}
	If a randomized mechanism $\mathcal{M}: \mathcal{X}^n\rightarrow \mathcal{R}$ is $(\alpha,\epsilon)$-RDP, then $\mathcal{M}$ is $(\epsilon+\log(1/\delta)/(\alpha-1), \delta)$-DP, $\forall \delta\in(0,1)$.
\end{lemma}

\begin{proof}
	Let's consider the Gaussian mechanism at the $t$-th iteration, as
	\begin{equation}
	\mathcal{M}_t = g^t(\mathbf{w}_{k^t}^t;\xi_{k^t}^t) + \mathcal{N}(0, \sigma^2\mathbf{I}).
	\end{equation}
	Since all functions are $G$-Lipschitz, we have the $l_2$-sensitivity bound $\triangle_2 = \|g(\mathbf{w}_{k^t}^t;\xi_{k^t}^t) - g(\mathbf{w}_{k^t}^{t};\xi_{k^t}^{'t})\|_2/B\leq 2G/B$. Thus, according to Lemma~\ref{lem:gau}, $\mathcal{M}_t$ is $(\alpha, 4G^2\alpha /(2B^2\sigma^2))$-RDP. 
	
	For the randomly sampling procedure $\mathcal{S}_{\gamma}^{\textsf{wo}}$ which is performed at the beginning of each iteration, we have the subsample rate $\gamma\leq\frac{B}{Kn_{(1)}}$, where $n_{(1)}$ is the size of the smallest dataset among $K$ workers. Then, the mechanism provides at least $(\alpha, 20G^2\alpha/(K^2n_{(1)}^2\sigma^2))$-RDP when $\sigma^2/\triangle_2^2\geq 1.5$. 
	After $T$ iterations, by sequential composition from Lemma~\ref{lem:seq_com}, we observe that the output of \adpdpsgd is $(\alpha, 20G^2T\alpha/(K^2n_{(1)}^2\sigma^2))$-RDP.
	
	Then, using the connection between RDP to $(\epsilon,\delta)$-DP from Lemma~\ref{lem:toDP}, we obtain
	\begin{equation}
	20G^2T\alpha/(K^2n_{(1)}^2\sigma^2)+\log(1/\delta)/(\alpha-1)=\epsilon.
	\end{equation}
	Let $\alpha=\log(1/\delta)/((1-\mu)\epsilon)+1$, we have
	\begin{equation}\label{eq:sigma}
	\sigma^2=\frac{20G^2T\alpha}{K^2n_{(1)}^2\epsilon\mu}\geq\frac{6G^2}{B^2},
	\end{equation}
	which implies that
	$\epsilon\leq \frac{10B^2T\alpha}{3K^2n_{(1)}^2\mu}$.
	
	In addition, via Lemma~\ref{lem:gau}, we need
	$\alpha \leq \log(1/(\gamma(1+\sigma^2/\triangle_2^2)))$,
	which implies that $\alpha\leq\log\left(K^3n_{(1)}^3\epsilon\mu/(K^2n_{(1)}^2\epsilon\mu B+5T\alpha B^3)\right)$.
	
	Thus, the output of \adpdpsgd is $(\epsilon,\delta)$-differentially private for the above value of $\sigma^2$.
\end{proof}

\end{document}